\documentclass[final,1p,times]{elsarticle}

\usepackage{graphicx}
\usepackage{hyperref}
\usepackage{xcolor}
\usepackage{textcomp}
\usepackage{bbding}
\usepackage{amsfonts}
\usepackage{amsmath}
\usepackage{amssymb}
\usepackage{mathtools}
\usepackage{mwe}
\usepackage{pgf}
\usepackage{tikz}
\usetikzlibrary{arrows, automata, shapes, petri, positioning, calc}
\usepackage{pgfplots}
\usepackage{rotating}
\usepackage{float}
\usepackage{multirow}
\usepackage[export]{adjustbox}
\usepackage{amsthm}

\def\bag{{\cal B}}
\def\prenet#1#2{\,\ensuremath{\stackrel{#1}{\bullet}\!{#2}}}
\def\postnet#1#2{\ensuremath{{#2}\!\kern-.05ex\stackrel{#1}{\bullet}}\,}
\def\bplus{\uplus}
\def\bmin{\setminus}
\def\pre#1{\ensuremath{\bullet{#1}}}
\def\post#1{\ensuremath{{#1}\kern-.05ex\bullet}}
\def\mi#1{\mathit{#1}}
\def\nomove{\gg}
\def\tproj{\!\!\upharpoonright}

\newcommand{\Nat}{\ensuremath{\mathrm{I\kern-1.5pt N}}}

\hypersetup{
	colorlinks,
	linkcolor={red!50!black},
	citecolor={blue!60!black},
	urlcolor={blue!80!black}
}

\journal{Information Systems}

\newtheorem{definition}{Definition}
\newtheorem{proposition}{Proposition}
\newtheorem{lemma}{Lemma}
\newtheorem{theorem}{Theorem}

\begin{document}

\begin{frontmatter}



\title{Conformance Checking over Uncertain Event Data}


\author{Marco Pegoraro\corref{cor1}}
\cortext[cor1]{Corresponding author}
\ead{pegoraro@pads.rwt-aachen.de}
\ead[url]{http://mpegoraro.net/}
\author{Merih Seran Uysal\corref{}}
\author{Wil M.P. van der Aalst\corref{}}
\address{Chair of Process and Data Science (PADS), Department of Computer Science, RWTH Aachen University, Aachen, Germany}

\begin{abstract}
The strong impulse to digitize processes and operations in companies and enterprises have resulted in the creation and automatic recording of an increasingly large amount of process data in information systems. These are made available in the form of \emph{event logs}. Process mining techniques enable the process-centric analysis of data, including automatically discovering process models and checking if event data conform to a given model. In this paper, we analyze the previously unexplored setting of uncertain event logs. In such event logs uncertainty is recorded explicitly, i.e., the time, activity and case of an event may be unclear or imprecise. In this work, we define a taxonomy of uncertain event logs and models, and we examine the challenges that uncertainty poses on process discovery and conformance checking. Finally, we show how upper and lower bounds for conformance can be obtained by aligning an uncertain trace onto a regular process model.
\end{abstract}

%


\end{frontmatter}


\section{Introduction}\label{sec:introduction}
Over the last decades, the concept of \emph{process} has become more and more central in formally describing the activities of businesses, companies and other similar entities, structured in specific steps and phases. A process is thus defined as a well-structured set of activities, potentially performed by multiple actors (\emph{resources}), which contribute to the completion of a specific task or to the achievement of a specific goal. In this context, a very important notion is the concept of \emph{case}, that is, a single instance of a process. For example, in a healthcare process, a case may be a single hospitalization of a patient, or the patient themself; if the process belongs to a credit institution, a case may be a loan application from a customer, and so on. The case notion allows us to define a process as a procedure that defines the steps needed to handle cases from inception to completion. A \emph{process model} defines such a procedure, and can be expressed in a number of different formalisms (transition systems, Petri nets, BPMN and UML diagrams, and many more). Consequently, the study and adoption of analysis techniques specifically customized to deal with process data and process models has enabled the bridging of business administration and data science and the development of dedicated disciplines like \emph{business intelligence} and \emph{Business Process Management} (BPM).

The processes that govern the innards of business companies are increasingly supported by software tools. Performing specific activities is both aided and recorded by \emph{Process-Aware Information Systems} (PAISs), which support the definition and management of processes. The information regarding the execution of processes can then be extracted from PAISs in the form of an \emph{event log}, a database or file containing the digital footprint of the operations carried out in the context of the execution of a process and recorded as \emph{events}. Event logs can vary in form, and contain differently structured information depending on the information system that enacted data collection in the organization. Although many different event attributes can be recorded, it is typically assumed that three basic features of an event are available in the log: the time in which the event occurred, the activity that has been performed, and the case identifier to which the event belong. This last attribute allows to group events in clusters belonging to the same case, and these resulting clusters (usually organized in sequences sorted by timestamp) are called \emph{process traces}. The discipline of \emph{process mining} is concerned with the automatic analysis of event logs, with the goal of extracting knowledge regarding e.g. the structure of the process, the conformity of events to a specific normative process model, the performance of the agents executing the process, the relationships between groups of actors in the process.

In this paper, we will consider the analysis of a specific class of event logs: logs that contain \emph{uncertain event data}. Uncertain events are recordings of executions of specific activities in a process which are enclosed with an indication of uncertainty in the event attributes. Specifically, we consider the case where the attributes of an event are not recorded as a precise value but as a range or a set of alternatives.

Uncertain event data are common in practice, but uncertainty is often not explicit. The \emph{Process Mining Manifesto}~\cite{van2011process} describes a fundamental property of event data as \emph{trustworthiness}, the assumption that the recorded data can be considered correct and accurate. In a general sense, uncertainty -- as defined here -- is an explicit absence of trustworthiness, with an indication of uncertainty recorded together with the event data. In the taxonomy of event data proposed in the Manifesto, the logs at the two lower levels of quality frequently lack trustworthiness, and thus can be uncertain. This encompasses a wide range of processes, such as event logs of document and product management systems, error logs of embedded systems, worksheets of service engineers, and any process recorded totally or partially on paper.
There are many possible causes of uncertainty:
\begin{itemize}
	\item \emph{Incorrectness}. In some instances, the uncertainty is simply given by errors that occurred while recording the data themselves. Faults of the information system, or human mistakes in a data entry phase can all lead to missing or altered event data that can be subsequently modeled as uncertain event data.
	\item \emph{Coarseness}. Some information systems have limitations in their way of recording data - often tied to factors like the precision of the data format - such that the event data can be considered uncertain. A typical example is an information system that only records the date, but not the time, of the occurrence of an event: if two events are recorded in the same day, the order of occurrence is lost. This is an especially common circumstance in the processes that are, partially or completely, recorded on paper and then digitalized. Another factor that can lead to uncertainty in the time of recording is the information system being overloaded and, thus, delaying the recording of data. This type of uncertainty can also be generated by the limited sensibility of a sensor.
	\item \emph{Ambiguity}. In some cases, the data recorded is not an identifier of a certain event attribute; in these instances, the data needs to be interpreted, either automatically or manually, in order to obtain a value for the event attribute. Uncertainty can arise if the meaning of the data is ambiguous and cannot be interpreted with precision. Examples include data in the form of images, text, or video.
\end{itemize}
These factors cause the presence of \emph{implicit} uncertainty in the event log. It is important to note that, in order to be analyzed, these indications of imprecision or incorrectness have to be translated into \emph{explicit} uncertainty. Explicit uncertainty is contained directly in the event log in the form of event attributes. It is possible to think of explicit uncertainty as metadata complementing the information regarding events. This metadata describes the type and magnitude of the imprecision affecting some event attributes, which might be part of the control-flow perspective or an additional data perspective present in the event log.

Aside from the possible causes, we can individuate other types of uncertain event logs based on the frequency of uncertain data. Uncertainty can be \emph{infrequent}, when a specific attribute is only seldomly recorded together with explicit uncertainty; the uncertainty is rare enough that uncertain events can be considered outliers. Conversely, \emph{frequent} uncertain behavior of the attribute is systematic, pervasive in a high number of traces, and thus not to be considered an outlier. The uncertainty can be considered part of the process itself. These concepts are not meant to be formal, and are laid out to distinguish between logs that are still processable regardless of the uncertainty, and logs where the uncertainty is too invasive to analyze them with existing process mining techniques.

In some contexts, the causes of uncertainty in event data can be resolved at the source, by acting directly on the process and on the tools supporting operations within it. For instance, a natural way to eliminate uncertainty in data recording is to automate tasks within the process as much as possible. This is a popular solution in applications like industry and manufacturing, where the actual tasks already involve machinery or automated systems. While we discuss the possibility of employing automation tools in processes where the majority of agents are humans in Section~\ref{sec:related_unc}, supporting a process with automation or auditing software that oversees and records the actions of agents is very challenging, for both technical and ethical reasons. Automatically recording data across different platforms, formats, and information systems by different producers and hosted by different service providers is often unfeasible, and there are legal reasons that might prevent it, such as breaches of confidentiality. In some jurisdictions, legally valid documents must be on paper, making real-time automatic data recording outright impossible. For these reasons, while eliminating uncertainties through automation is an advisable choice, analyzing data containing a description of such uncertainties is sometimes the only analysis technique that can deliver approximate but trustworthy results.

The diagram in Figure~\ref{fig:schema} shows an overview of the main elements of process mining over uncertainty. The schema shows some additional elements with respect to classical process mining: we can see that we can combine raw process data from information systems (containing implicit uncertainty) with domain knowledge provided by a process expert to obtain an uncertain event log, which contains explicit uncertainty. The data in an uncertain event log can be abstracted in a graph representation, which enables the inspection of its causes. Lastly, the graph representations also allows to perform the tasks of process discovery and conformance checking on uncertain event data.

\begin{figure}
	\centering
	\includegraphics[width=1\columnwidth]{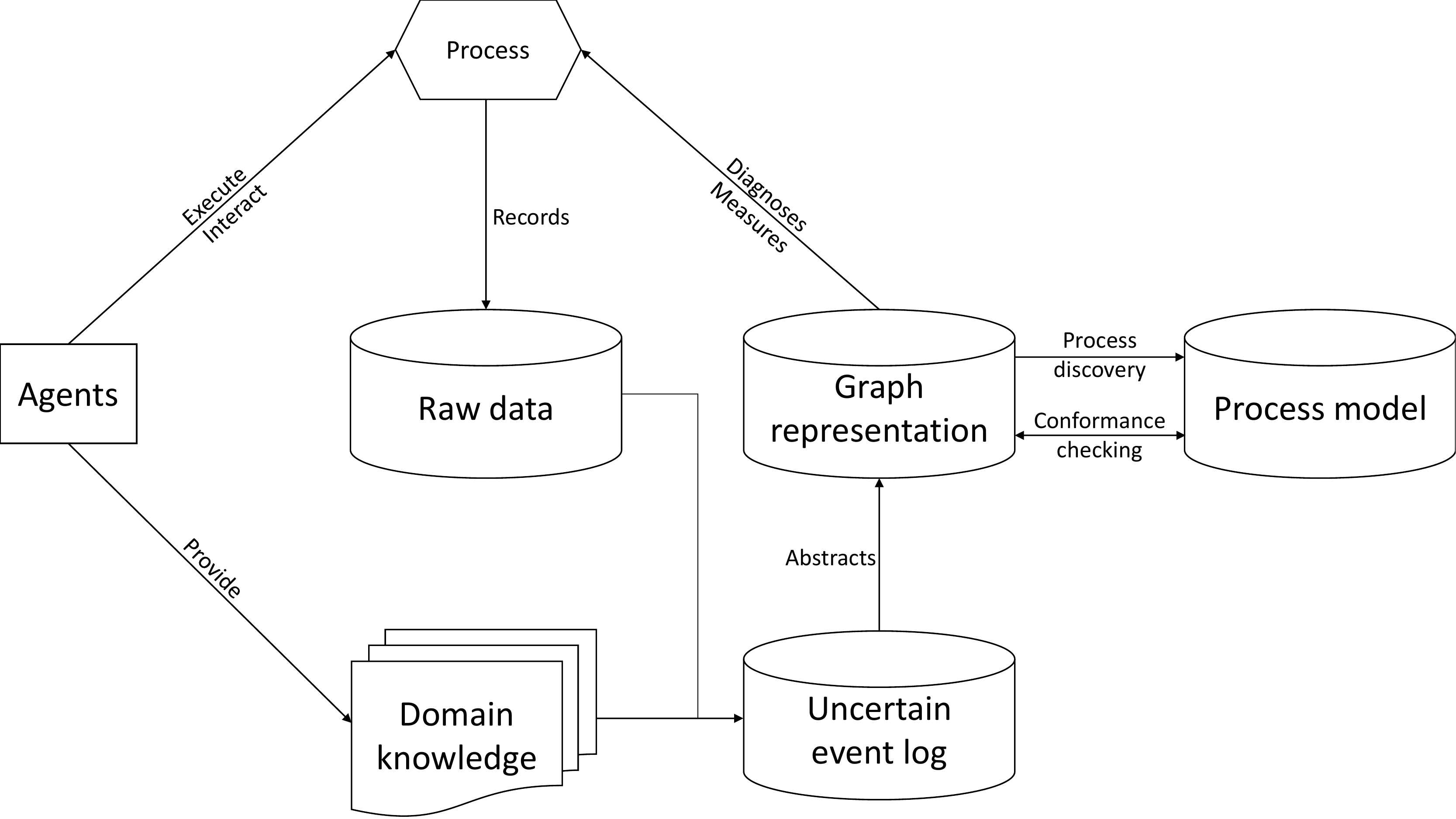}
	\caption{The overall schema for process mining over uncertain event data.}
	\label{fig:schema}
\end{figure}

In this paper, we propose a taxonomy of the different types of explicit uncertainty in process mining, together with a formal, mathematical formulation. As an example of practical application, we will consider the case of conformance checking~\cite{carmona2018conformance}, and we will apply it to uncertain data by assessing what are the upper and lower bounds on the conformance score for possible values of the attributes in an uncertain trace.

The main drivers behind this work is to provide the means to treat uncertainty as a relevant part of a process; thus, we aim not to filter it out but to model and explain it. In conclusion, there are two novel aspects regarding uncertain data that we intend to address in this work. The first novelty is the \emph{explicitness of uncertainty}: we work with the underlying assumption that the actual value of the uncertain attribute, while not directly provided, is described formally. This is the case when meta-information about the uncertainty in the attribute is available, either deduced from the features of the information system(s) that record the logs or included in the event log itself. Note that, as opposed to all previous work on the topic, the fact that uncertainty is explicit in the data means that the concept of uncertain behavior is completely separated from the concept of infrequent behavior. The second novelty is the \emph{explicit modeling of uncertainty}: we consider uncertainty part of the process. Instead of filtering or cleaning the log, we introduce the uncertainty perspective in process mining by extending the currently available techniques to incorporate it.

In summary, the paper aims to discuss the following research questions:
\begin{itemize}
	\item What is uncertainty, and in which ways can it manifest within event data?
	\item How can uncertain event data be processed to retain all information contained within it?
	\item How can we exploit this representation to solve classical process mining tasks, such as conformance checking?
\end{itemize}

The rest of this paper is organized as follows. Section~\ref{sec:taxonomy} proposes a taxonomy of the different possible types of uncertain process data. Section~\ref{sec:definitions} contains the formal definitions needed to manage uncertainty. Section~\ref{sec:unc} presents the main contribution of this paper, a framework able to describe an array of types and classifications of uncertain behavior. Section~\ref{sec:conformance} describes a practical application of process mining over uncertain event data, the case of conformance checking through alignments. Section~\ref{sec:experiments} shows experimental results on computing conformance checking scores for synthetic uncertain data, as well as a case of application on real-life data. Section~\ref{sec:related} discusses previous and related work on the management of uncertain data and on the topic of conformance checking. Finally, Section~\ref{sec:conclusion} concludes the paper and discusses future work.

\section{A Taxonomy of Uncertain Event Data}\label{sec:taxonomy}
The goal of this section of the paper is to propose a categorization of the different types of uncertainty that can appear in process mining. In process management, a central concept is the distinction between the data perspective (the event log) and the behavioral perspective (the process model). The first one is a static representation of process instances, the second summarizes the behavior of a process. Both can be extended with a concept of explicit uncertainty: this concept also implies an extension of the process mining techniques that have currently been implemented.

In this paper, we will focus on uncertainty in event data, rather than applying the concept of uncertainty to models. Specifically, we will consider computing the conformance score of uncertain process data on classical models, extending the approach shown in~\cite{pegoraro2019mining}. An application of process discovery in the setting of uncertain event data has been presented in~\cite{pegoraro2019discovering}.

We can individuate two different notions of uncertainty:

\begin{itemize}
	\item \emph{Strong uncertainty}: the possible values for the attributes are known, but the probability that the attribute will assume a certain instantiation is unknown or unobservable.
	\item \emph{Weak uncertainty}: both the possible values of an attribute and their respective probabilities are known.
\end{itemize}

In the case of a discrete attribute, the strong notion of uncertainty consists on a set of possible values assumed by the attribute. In this case, the probability for each possible value is unknown. Vice-versa, in the weak uncertainty scenario we also have a discrete probability distribution defined on that set of values.
In the case of a continuous attribute, the strong notion of uncertainty can be represented with an interval for the variable. Notice that an interval does not indicate a uniform distribution; there is no information on the likelihood of values in it. Vice-versa, in the weak uncertainty scenario we also have a probability density function defined on a certain interval. Table~\ref{table:unctypes} summarizes these concepts. This leads to very simple representations of explicit uncertainty.

In this paper, we consider only the control flow and time perspective of a process -- namely, the attributes of the events that allow us to discover a process model. These are the unique identifier of a process instance (case ID), the timestamp (often represented by the distance from a fixed origin point, e.g. the \emph{Unix Epoch}), and the activity identifier of an event. Case IDs and activities are values chosen from a finite set of possible values; they are discrete variables. Timestamps, instead, are represented by numbers and thus are continuous variables.

\pgfmathdeclarefunction{gauss}{2}{%
	\pgfmathparse{1/(#2*sqrt(2*pi))*exp(-((x-#1)^2)/(2*#2^2))}%
}

\begin{table}[H]
	\centering
	\begin{tabular}{|c|c|c|}
		\hline
		& \textbf{Weak uncertainty} & \textbf{Strong uncertainty} \\ \hline
		\textbf{Discrete data}   &
		\begin{tabular}[c]{@{}c@{}} Discrete probability distribution\\ \\
			\begin{minipage}{.25\textwidth}
				\resizebox {\textwidth} {!} {
					\begin{tikzpicture}
					\begin{axis}[ybar interval, ymax=55,ymin=0, minor y tick num = 3]
					\addplot coordinates { (0, 5) (5, 35) (10, 50) (15, 30) (20, 15) (25, 0) };
					\end{axis}
					\end{tikzpicture}
				}
			\end{minipage}
		\end{tabular} &
		\begin{tabular}[c]{@{}c@{}} Set of possible values \\ \\ $\{ x, y, z, \dots \}$
		\end{tabular} \\ \hline
		\textbf{Continuous data} &
		\begin{tabular}[c]{@{}c@{}} Probability density function \\ \\
			\begin{minipage}{.25\textwidth}
				\resizebox {\textwidth} {!} {
					\begin{tikzpicture}
					\begin{axis}[every axis plot post/.append style={
						mark=none,domain=-2:3,samples=50,smooth}, 
						axis x line*=bottom, 
						axis y line*=left, 
						enlargelimits=upper]
					\addplot {gauss(1,0.5)};
					\end{axis}
					\end{tikzpicture}
				}
			\end{minipage}
		\end{tabular} &
		\begin{tabular}[c]{@{}c@{}} Interval \\ \\ $\{ x \in \mathbb{R} \mid a \leq x \leq b \}$
		\end{tabular} \\ \hline
	\end{tabular}
	\caption{The four different types of uncertainty.}
	\label{table:unctypes}
\end{table}

We will also describe an additional type of uncertainty, which lies on the event level rather than the attribute level:

\begin{itemize}
	\item \emph{Indeterminate event}: the event may have not taken place even though it was recorded in the event log. Indeterminate events are indicated with a ? symbol, while determinate (regular) events are marked with a ! symbol.
\end{itemize}

\begin{table}[!htb]
	\centering
	\caption{An example of a strongly uncertain trace. For the sake of clarity, the timestamp field only reports dates.}
	\label{table:uncertaintracestrong}
	\begin{tabular}{cccc}
		\textbf{Case ID}                  & \textbf{Timestamp}                                                                                      & \textbf{Activity}                & \multicolumn{1}{l}{\textbf{Indet. event}} \\ \hline
		\multicolumn{1}{|c|}{\{ID327, ID412\}}    & \multicolumn{1}{c|}{2011-12-05}                                                                   & \multicolumn{1}{c|}{A}           & \multicolumn{1}{c|}{!}                    \\ \hline
		\multicolumn{1}{|c|}{ID327}           & \multicolumn{1}{c|}{2011-12-07}                                                                   & \multicolumn{1}{c|}{\{B, C, D\}} & \multicolumn{1}{c|}{!}                    \\ \hline
		\multicolumn{1}{|c|}{ID327}           & \multicolumn{1}{c|}{\begin{tabular}[c]{@{}c@{}}{[}2011-12-06, 2011-12-10{]}\end{tabular}} & \multicolumn{1}{c|}{D}           & \multicolumn{1}{c|}{?}                    \\ \hline
		\multicolumn{1}{|c|}{ID327}           & \multicolumn{1}{c|}{2011-12-09}                                                                   & \multicolumn{1}{c|}{\{A, C\}}    & \multicolumn{1}{c|}{!}                    \\ \hline
		\multicolumn{1}{|c|}{\{ID327, ID412, ID573\}} & \multicolumn{1}{c|}{2011-12-11}                                                                & \multicolumn{1}{c|}{E}           & \multicolumn{1}{c|}{?}                    \\ \hline
	\end{tabular}
\end{table}

\begin{table}[!htb]
	\caption{An example of a weakly uncertain trace. For the sake of clarity, the timestamp field only reports dates.}
	\label{table:uncertaintraceweak}
	\centering
	\begin{tabular}{cccc}
		\textbf{Case ID}                       & \textbf{Timestamp}                          & \textbf{Activity}                     & \multicolumn{1}{l}{\textbf{Indet. event}} \\ \hline
		\multicolumn{1}{|c|}{\{ID313:0.9, ID370:0.1\}} & \multicolumn{1}{c|}{2011-12-05}       & \multicolumn{1}{c|}{A}                & \multicolumn{1}{c|}{!}                    \\ \hline
		\multicolumn{1}{|c|}{ID313}                & \multicolumn{1}{c|}{2011-12-07}       & \multicolumn{1}{c|}{\{B:0.7, C:0.3\}} & \multicolumn{1}{c|}{!}                    \\ \hline
		\multicolumn{1}{|c|}{ID313}                & \multicolumn{1}{c|}{$\mathcal{N}$(2011-12-08, 2)} & \multicolumn{1}{c|}{D}                & \multicolumn{1}{c|}{?:0.5}                \\ \hline
		\multicolumn{1}{|c|}{ID313}                & \multicolumn{1}{c|}{2011-12-09}       & \multicolumn{1}{c|}{\{A:0.2, C:0.8\}} & \multicolumn{1}{c|}{!}                    \\ \hline
		\multicolumn{1}{|c|}{\{ID313:0.4, ID370:0.6\}} & \multicolumn{1}{c|}{2011-12-11}    & \multicolumn{1}{c|}{E}                & \multicolumn{1}{c|}{?:0.7}                \\ \hline
	\end{tabular}
\end{table}

Examples of strongly and weakly uncertain traces are shown in Tables~\ref{table:uncertaintracestrong} and~\ref{table:uncertaintraceweak} respectively. Additionally, we present a time diagram of the trace in Table~\ref{table:uncertaintracestrong}: this representation shows the time relationship between events in the trace in absolute scale. This diagram is shown in Figure~\ref{fig:ganttintr}

\begin{figure}[h!]
	\centering
	\includegraphics[width=.9\columnwidth]{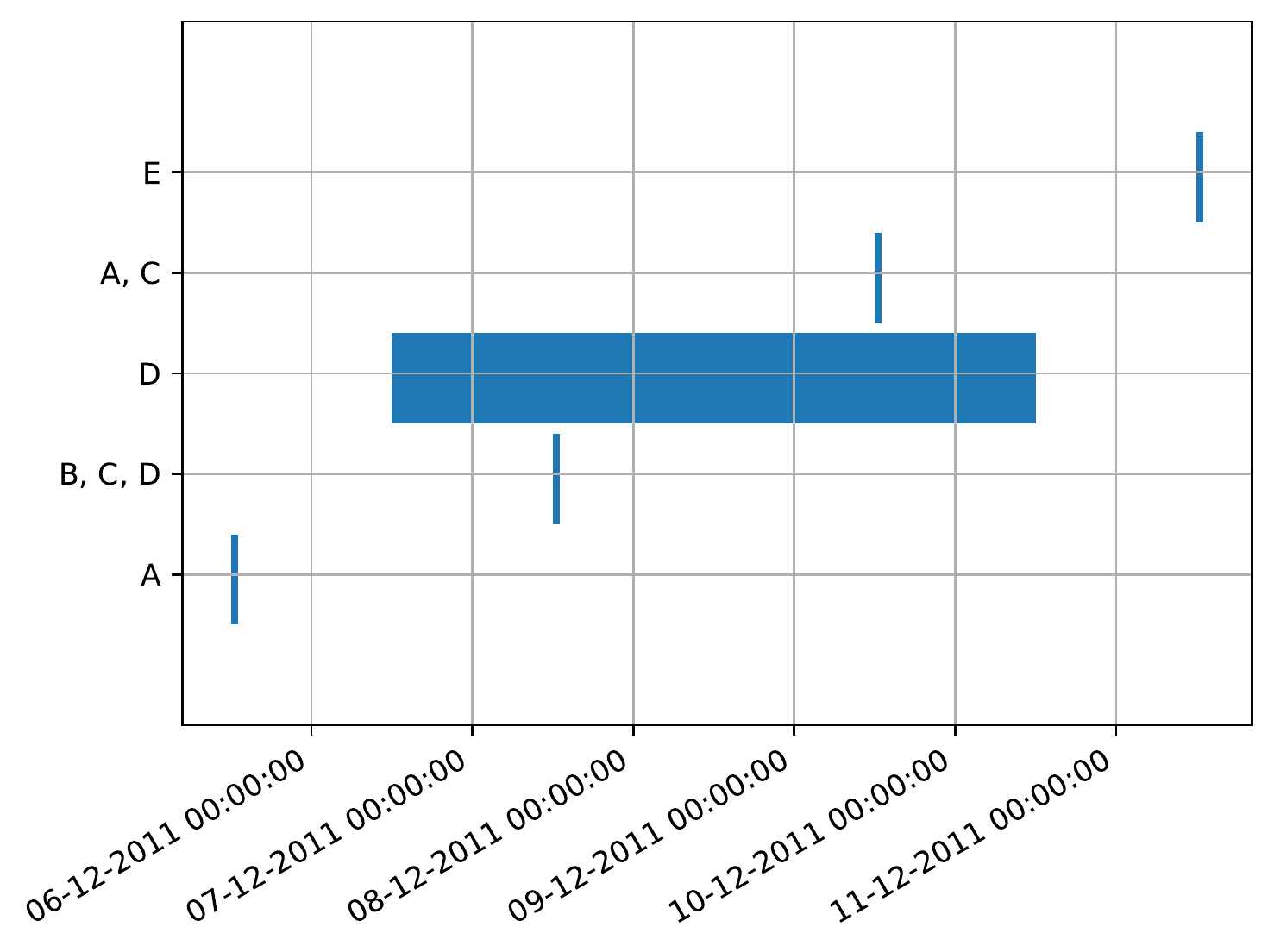}
	\caption{Time diagram of the trace in Table~\ref{table:uncertaintracestrong}. This diagram shows the time information of an uncertain trace in an absolute scale. Note that some types of uncertainty (namely, indeterminate events and uncertainty on case IDs) are not depicted.}
	\label{fig:ganttintr}
\end{figure}


The taxonomy presented in this section is summarized in Table~\ref{table:taxonomy}. This table encodes all types of uncertainty illustrated here. Through this taxonomy, we can indicate the types of uncertainty that might affect an uncertain event log.

\begin{table}[]
	\caption{Summary of the types of uncertainty that can affect a log over the attributes of its events. The last column provides an encoding for each type of uncertainty.}
	\label{table:taxonomy}
	\centering
	\begin{tabular}{|c|c|c|c|}
		\hline
		\textbf{Attribute}                       & \textbf{Attribute type}                          & \textbf{Uncertainty type}                     & \textbf{Encoding} \\ \hline
		\multirow{2}{*}{\begin{tabular}[c]{@{}c@{}}Event\\ (indeterminacy)\end{tabular}} & \multirow{2}{*}{Discrete}   & Weak   & $\text{[E]}_\mathbb{W}$   \\ \cline{3-4} 
		&                             & Strong & $\text{[E]}_\mathbb{S}$   \\ \hline
		\multirow{2}{*}{Case}                                                            & \multirow{2}{*}{Discrete}   & Weak   & $\text{[C]}_\mathbb{W}$   \\ \cline{3-4} 
		&                             & Strong & $\text{[C]}_\mathbb{S}$   \\ \hline
		\multirow{2}{*}{Activity}                                                        & \multirow{2}{*}{Discrete}   & Weak   & $\text{[A]}_\mathbb{W}$   \\ \cline{3-4} 
		&                             & Strong & $\text{[A]}_\mathbb{S}$   \\ \hline
		\multirow{2}{*}{Timestamp}                                                       & \multirow{2}{*}{Continuous} & Weak   & $\text{[T]}_\mathbb{W}$   \\ \cline{3-4} 
		&                             & Strong & $\text{[T]}_\mathbb{S}$   \\ \hline
		\multirow{4}{*}{Other attribute}                                                       & \multirow{2}{*}{Discrete}   & Weak   & $\text{[ATD]}_\mathbb{W}$ \\ \cline{3-4} 
		&                             & Strong & $\text{[ATD]}_\mathbb{S}$ \\ \cline{2-4} 
		& \multirow{2}{*}{Continuous} & Weak   & $\text{[ATC]}_\mathbb{W}$ \\ \cline{3-4} 
		&                             & Strong & $\text{[ATC]}_\mathbb{S}$ \\ \hline
	\end{tabular}
\end{table}

More types of uncertainty can be combined to describe an uncertain event log. For example, an event log with strong uncertainty on events, activities and timestamps would be an $\text{[E, A, T]}_\mathbb{S}$-type log. An uncertain log can also be characterized by different types of uncertainty on different attributes: a log with strong uncertainty on events and weak uncertainty on activities is a $\text{[E]}_\mathbb{S}\text{[A]}_\mathbb{W}$-type log.

In the next section, we will describe these different types of uncertainty in a mathematical framework that will, in turn, enable process mining analyses on uncertain event data.

\section{Preliminaries}\label{sec:definitions}
Let us introduce some preliminary definitions in order to describe uncertainty in process mining in a formal way. These definitions will provide the means to represent the behavior contained in uncertain data, and enable process mining tasks such as process discovery and conformance checking.

\subsection{Basic Definitions}
Firstly, we will define some basic mathematical structures.

\begin{definition}[Power Set]
	The \emph{power set} of a set $A$ is the set of all possible subsets of $A$, and is denoted with $\mathcal{P}(A)$. $\mathcal{P}_{NE}(A)$ denotes the set of all the non-empty subsets of $A$: $\mathcal{P}_{NE}(A) = \mathcal{P}(A)\setminus\{\emptyset\}$.
\end{definition}

\begin{definition}[Multiset]
	A \emph{multiset} is an extension of the concept of set that keeps track of the cardinality of each element. $\bag(A)$ is the set of all multisets over some set $A$. Multisets are denoted with square brackets, e.g. $b_1 = [~]$
	(the empty multiset), $b_2 = [a, a, b]$, $b_3 = [a, b, c]$, $b_4 = [a, b, c, a, a, b]$ are all multisets over $A = \{a, b, c\}$. In multiset the order of representation of the elements is irrelevant, and they can also be denoted with the cardinality of their elements, e.g. $b_4 = [a, b, c, a, a, b] = [a^3, b^2, c]$. We denote with $b(x)$ the cardinality of element $x \in A$ in $b$, e.g. $b_4(a) = 3$, $b_4(c) = 1$, and $b_4(d) = 0$.
	
	We can extend to multiset standard set operators such as membership (e.g. $a \in b_2$ and $c \notin b_2$), union (e.g. $b_2 \uplus b_3 = b_4$), difference (e.g. $b_4 \bmin b_3 = b_2$) and total cardinality (e.g. $|b_4| = 6$).
\end{definition}

\begin{definition}[Sequence, Subsequence and Permutation]
	Given a set $X$, a finite \emph{sequence} over $X$ of length $n$ is a function $s \in X^* : \{1, \dots, n\} \rightarrow X$, and it is written as $s = \langle s_1, s_2, \dots, s_n\rangle$. We denote with $\langle~\rangle$ the empty sequence, the sequence with no elements and of length 0. Over the sequence $s$ we define $|s| = n$, $s[i] = s_i$ and $x \in s \Leftrightarrow \exists_{1 \leq i \leq n} \ s = s_i$. The concatenation between two sequences is denoted with $\langle s_1, s_2, \dots, s_n\rangle \cdot \langle s'_1, s'_2, \dots, s'_m\rangle = \langle s_1, s_2, \dots, s_n, s'_1, s'_2, \dots, s'_m\rangle$. Given two sequences $s = \langle s_1, s_2, \dots, s_n\rangle$ and $s' = \langle s'_1, s'_2, \dots, s'_m\rangle$, $s'$ is a \emph{subsequence} of $s$ if and only if there exists a sequence of strictly increasing natural numbers $\langle i_1, i_2, \dots, i_m \rangle$ such that $\forall_{1 \leq j \leq m} \ s_{i_j} = s'_j$. We indicate this with $s' \subseteq s$. A \emph{permutation} of the set $X$ is a sequence $x_\mathcal{S}$ that contains all elements of $X$ without duplicates: $x_\mathcal{S} \in X$, $X \in x_\mathcal{S}$, and for all $1 \leq i \leq |x_\mathcal{S}|$ and for all $1 \leq j \leq |x_\mathcal{S}|$, $x_\mathcal{S}[i] = x_\mathcal{S}[j] \rightarrow i = j$. We denote with $\mathcal{S}_X$ all such permutations of set $X$.
\end{definition}

\begin{definition}[Sequence Projection]\label{def:sproj}
	Let $X$ be a set and $Q\subseteq X$ one of its subsets.~$\tproj_{Q} \colon X^* \rightarrow Q^*$ is the sequence projection function and is defined recursively: $\langle~\rangle \tproj_{Q} = \langle~\rangle$ and for $\sigma \in X^*$ and $x\in X$:
	$$(\langle x \rangle \cdot \sigma)\tproj_{Q} = \begin{cases}
	\sigma \tproj_{Q} & \mbox{if} \ x \not\in Q\\
	\langle x \rangle \cdot \sigma \tproj_{Q} & \mbox{if} \ x \in Q
	\end{cases}$$
\end{definition}

For example, $\langle y,z,y\rangle \tproj_{\{x,y\}} = \langle y,y\rangle$.

\begin{definition}[Applying Functions to Sequences]\label{def:funseq}
	Let $f \colon X \not\rightarrow Y$ be a partial function.
	$f$ can be applied to sequences of $X$ using the following recursive definition: $f(\langle~\rangle) = \langle~\rangle$ and for $\sigma \in X^*$ and $x\in X$:
	$$f(\langle x \rangle \cdot \sigma) = \begin{cases}
	f(\sigma) & \mbox{if} \ x \not\in dom(f)\\
	\langle f(x) \rangle \cdot f(\sigma) & \mbox{if} \ x \in dom(f)
	\end{cases}$$
\end{definition}

Next, so as to manage the possible different orders between events in a trace with uncertain timestamps, we introduce formalisms to denote strict partial orders.

\begin{definition}[Transitive Relation and Correct Evaluation Order]\label{def:tr_rel}
	Let $X$ be a set of objects and $R$ be a binary relation $R \subseteq X \times X$. $R$ is \emph{transitive} if and only if for all $x, x', x'' \in X$ we have that $(x, x') \in R \wedge (x', x'') \in R \Rightarrow (x, x'') \in R$. A \emph{correct evaluation order} is a permutation $s \in \mathcal{S}_X$ of the elements of the set $X$ such that for all $1 \leq i < j \leq |s|$ we have that $(s[j], s[i]) \not\in R$.
\end{definition}

\begin{definition}[Strict Partial Order]\label{def:st-par-ord}
	Let $S$ be a set of objects. Let $s, s' \in S$. A \emph{strict partial order} $\prec$ over $S$ is a binary relation that satisfies the following properties:
	\begin{itemize}
		\item Irreflexivity: $s \prec s$ is false.
		\item Transitivity: see Definition~\ref{def:tr_rel}.
		\item Antisymmetry: $s \prec s'$ implies that $s' \prec s$ is false. Implied by irreflexivity and transitivity\emph{~\cite{flavska2007transitive}}.
	\end{itemize}
\end{definition}

\begin{definition}[Directed Graph]
	A \emph{directed graph} $G$ is a tuple $(V, E)$ where $V$ is the set of vertices and $E \subseteq V \times V$ is the set of directed edges. The set $\mathcal{U}_G$ is the \emph{graph universe}. A \emph{path} in a directed graph $G = (V, E)$ is a sequence of vertices $p \in V$ such that for all $1<i<|p|-1$ we have that $(p_i, p_{i+1}) \in E$. We denote with $P_G$ the set of all such possible paths over the graph G. Given two vertices $v, v' \in V$, we denote with $p_G(v, v')$ the set of all paths beginning in $v$ and ending in $v'$: $p_G(v, v') = \{p \in P_G \mid p[1] = v \wedge p[|p|] = v'\}$. $v$ and $v'$ are \emph{connected} (and $v'$ is \emph{reachable} from $v$), denoted by $v \overset{G}{\mapsto} v'$, if and only if there exists a path between them in $G$: $p_G(v, v') \neq \emptyset$. Conversely, $v \overset{G}{\not\mapsto} v' \Leftrightarrow p_G(v, v') = \emptyset$. We omit the superscript $G$ if it is clear from the context. A directed graph $G$ is \emph{acyclic} if there exists no path $p \in P_G$ satisfying $p[1] = p[|p|]$.
\end{definition}

\begin{definition}[Topological Sorting]\label{def:topsort}
	Let $G = (V, E) \in \mathcal{U}_G$ be an acyclic directed graph. A \emph{topological sorting~\cite{kalvin1983generation}} $o_G \in \mathcal{S}_V$ is a permutation of the vertices of $G$ such that for all $1 \leq i < j \leq |o_G|$ we have that $o_G[j] \not\mapsto o_G[i]$. We denote with $\mathcal{O}_G \subseteq \mathcal{S}_V$ all such possible topological sortings over $G$.
\end{definition}

\begin{definition}[Transitive Reduction]
	A \emph{transitive reduction of a graph $G = (V, E) \in \mathcal{U}_G$~\cite{aho1972transitive}} is the function $\rho \colon \mathcal{U}_G \to \mathcal{U}_G$ such that for the graph $\rho(G) = (V, E_r)$ we have $E_r \subseteq E$ and every pair of vertices connected in $\rho(G)$ is not connected by any other path: for all $(v, v') \in E_r$, $p_G(v, v') = \{\langle v, v' \rangle\}$. $\rho(G)$ is the graph with the minimal number of edges that maintain the reachability between edges of $G$. The transitive reduction of a directed acyclic graph always exists and is unique\emph{~\cite{aho1972transitive}}.
\end{definition}

%

\subsection{Process Mining Definitions}

Let us now define the basic artifacts needed to perform process mining.

\begin{definition}[Universes]
	Let $\mathcal{U}_I$ be the set of all the \emph{event identifiers}. Let $\mathcal{U}_C$ be the set of all the \emph{case ID identifiers}. Let $\mathcal{U}_A$ be the set of all the \emph{activity identifiers}. Let $\mathcal{U}_T$ be the totally ordered set of all the \emph{timestamp identifiers}.
\end{definition}

\begin{definition}[Events and event logs]
	Let us denote with $\mathcal{E}_C = \mathcal{U}_I \times \mathcal{U}_C \times \mathcal{U}_A \times \mathcal{U}_T$ the universe of \emph{certain events}. A \emph{certain event log} is a set of events $L_C \subseteq \mathcal{E}_C$ such that every event identifier in $L_C$ is unique.
\end{definition}

\begin{definition}[Simple certain traces and logs]
	Let $\{ (e_1, c_1, a_1, t_1), (e_2, c_2, a_2, t_2), \dots, (e_n, c_n, a_n,\allowbreak t_n) \} \subseteq L_C$ be a set of certain events such that $c_1 = c_2 = \dots = c_n$ and $t_1 < t_2 < \dots < t_n$. A \emph{simple certain trace} is the sequence of activities $\langle a_1, a_2, \dots, a_n \rangle \in {\mathcal{U}_A}^*$ induced by such a set of events.  $\mathcal{T} = {\mathcal{U}_A}^*$ denotes the universe of certain traces. $L \in \bag(\mathcal{T})$ is a \emph{simple certain log}. We will drop the qualifier ``simple'' if it is clear from the context.
\end{definition}

As a preliminary application of process mining over uncertain event data, we will consider conformance checking. Starting from an event log and a process model, conformance checking verifies if the event data in the log conforms to the model, providing a diagnostic of the deviations. Conformance checking serves many purposes, such as checking if process instances follow a specific normative model, assessing if a certain execution log has been generated from a specific model, or verifying the quality of a process discovery technique.

The conformance checking algorithm that we are applying in this paper is based on \emph{alignments}. Introduced by Adriansyah~\cite{adriansyah2014aligning}, conformance checking through alignments finds deviations between a trace and a Petri net model of a process by creating a correspondence between the sequence of activities executed in the trace and the firing of the transitions in the Petri net. The following definitions are partially from~\cite{van2013decomposing}.

\begin{definition}[Petri Net]
	A Petri net is a tuple $N = (P,T,F)$ with $P$ the set of places, $T$ the set of transitions, $P \cap T = \emptyset$, and $F\subseteq (P \times T) \cup (T \times P)$ the flow relation. A Petri net $N = (P,T,F)$ defines a directed graph $(V, E)$ with vertices $V = P \cup T$ and edges $E = F$. A marking $M \in \bag(P)$ is a multiset of places.
\end{definition}

A marking defines the state of a Petri net, and indicates how many \emph{tokens} each place contains. For any $x \in P \cup T$, $\prenet{N}{x} = \{x'\mid (x',x)\in F\}$ denotes the set of input nodes and
$\postnet{N}{x} = \{x'\mid (x,x')\in F\}$ denotes the set of output nodes. We omit the superscript $N$ if it is clear from the context.

A transition $t \in T$ is \emph{enabled} in marking $M$ of net $N$, denoted as $(N,M)[t\rangle$, if each of its input places $\pre t$ contains at least one token. An enabled transition $t$ may \emph{fire}, i.e., one token is removed from each of the input places $\pre t$ and
one token is produced for each of the output places $\post t$.
Formally: $M' = (M\bmin \pre t)\bplus \post t$ is the marking resulting from firing enabled transition $t$ in marking $M$ of Petri net $N$.
$(N,M)[t\rangle (N,M')$ denotes that $t$ is enabled in $M$ and firing $t$ results in marking $M'$.

Let $\sigma_T = \langle t_1,t_2, \ldots, t_n \rangle \in T^*$ be a sequence of transitions.
$(N,M)[\sigma_T\rangle (N,M')$ denotes that there is a set of markings $M_0, M_1, \ldots, M_n$
such that $M_0 = M$, $M_n = M'$, and $(N,M_i)[t_{i+1}\rangle (N,M_{i+1})$ for $0 \leq i < n$.
A marking $M'$ is \emph{reachable} from $M$ if there exists a $\sigma_T$ such that $(N,M)[\sigma_T\rangle (N,M')$.

\begin{definition}[Labeled Petri Net]
	A labeled Petri net $N = (P,T,F,l)$ is a Petri net $(P,T,F)$ with labeling function $l \colon T \not\to {\cal U}_{A}$ where ${\cal U}_{A}$ is some universe of activity labels.
	Let $\sigma = \langle a_1,a_2, \ldots, a_n \rangle \in {{\cal U}_{A}}^*$ be a sequence of activities.
	$(N,M) [{\sigma} \rhd (N,M')$ if and only if there is a sequence $\sigma_T \in T^*$ such that
	$(N,M) [\sigma_T\rangle (N,M')$ and $l(\sigma_T)=\sigma$.
\end{definition}

If $t \notin \mi{dom}(l)$, it is called \emph{invisible}. To indicate invisible transitions, we use the placeholder symbol $\tau \notin {\cal U}_{A}$; for any invisible transition $t$ we define $l(t) = \tau$. An occurrence of visible transition $t \in \mi{dom}(l)$ corresponds to observable activity $l(t)$.

\begin{definition}[System Net]
	A system net is a triplet $\mi{SN} = (N,M_{\mi{init}},M_{\mi{final}})$ where
	$N = (P,T,F,l)$ is a labeled Petri net,
	$M_{\mi{init}} \in \bag(P)$ is the initial marking, and $M_{\mi{final}} \in \bag(P)$ is the final marking.
	${\cal U}_{\mi{SN}}$ is the \emph{universe of system nets}.
	Over a system net we define the following:
	\begin{itemize}
		\item $T_v(\mi{SN}) = \mi{dom}(l)$ is the set of \emph{visible transitions} in $\mi{SN}$,
		\item $A_v(\mi{SN}) = \mi{rng}(l)$ is the set of corresponding \emph{observable activities} in $\mi{SN}$,
		\item $T_v^u(\mi{SN}) = \{t \in T_v(\mi{SN}) \mid
		\forall_{t' \in T_v(\mi{SN})} \ l(t) = l(t') \Rightarrow  t = t' \}$ is the set of \emph{unique} visible transitions in $\mi{SN}$ (i.e., there are no other transitions having the same visible label),
		\item $A_v^u(\mi{SN}) = \{l(t) \mid t \in T_v^u(\mi{SN})\}$ is the set of corresponding \emph{unique} observable activities in $\mi{SN}$,
		\item $\phi(\mi{SN}) = \{\sigma \mid (N,M_{\mi{init}}) [{\sigma} \rhd (N,M_{\mi{final}})\}$ is the set of \emph{visible} traces starting in $M_{\mi{init}}$ and ending in $M_{\mi{final}}$, and
		\item $\phi_f(\mi{SN}) = \{\sigma_T \mid (N,M_{\mi{init}}) [{\sigma_T}\rangle (N,M_{\mi{final}})\}$
		is the corresponding set of complete firing sequences.
	\end{itemize}
\end{definition}

Figure~\ref{fig:esconformance} shows a system net with initial and final markings $M_{\mi{init}} = [start]$ and $M_{\mi{final}} = [end]$. Given a system net, $\phi(\mi{SN})$ is the set of all possible \emph{visible} activity sequences, i.e., the labels of complete firing sequences starting in $M_{\mi{init}}$ and ending in $M_{\mi{final}}$ projected onto the set of observable activities. Given the set of activity sequences $\phi(\mi{SN})$ obtainable via complete firing sequences on a certain system net, we can define a perfectly fitting event log as a set of traces which activity projection is contained in $\phi(\mi{SN})$.

\subsection{Conformance Checking Definitions}

The task of conformance checking consist in comparing an event log and a model, in order to assess the deviations of event data with respect to the expected behavior of the process. This is usually done to verify if the process conforms to a \emph{de iure} model designed by process experts, which describes how the process should ideally run. We will now describe a conformance checking technique, in order to extend it to the uncertain setting.

\begin{definition}[Perfectly Fitting Log]\label{def:perffitlog}
	Let $L \in \bag(\mathcal{T})$ be a certain event log and let $\mi{SN}=(N,M_{\mi{init}},M_{\mi{final}})\in {\cal U}_{\mi{SN}}$ be a system net.
	$L$ is perfectly fitting $\mi{SN}$ if and only if $\{\sigma \in L\} \subseteq \phi(\mi{SN})$.
\end{definition}

The definitions described so far allow us to build \emph{alignments} in order to compute the fitness of trace on a certain model. An alignment is a correspondence between a sequence of activities (extracted from the trace) and a sequence of transitions with the relative labels (fired in the model while replaying the trace). The first sequence indicates the ``moves in the log'' and the second indicates the ``moves in the model''. If a move in the model cannot be mimicked by a move in the log, then a ``$\nomove$'' (``no move'') appears in the top row; conversely, if a move in the log cannot be mimicked by a move in the model, then a ``$\nomove$'' (``no move'') appears in the bottom row.``no moves'' not corresponding to invisible transitions point to
deviations between the model and the log. A \emph{move} is a pair $(x,(y,t))$ where the first element refers to the log and the second element to the model. A ``$\nomove$'' in the first element of the pair indicates a move on the model, while a ``$\nomove$'' in the second element indicates a move on the log.

\begin{definition}[Legal Moves]
	Let $L \in \bag(\mathcal{T})$ be a certain event log, let $A \subseteq \mathcal{U}_A$ be the set of activity labels appearing in the event log, and let $\mi{SN} = (N, M_{\mi{init}},\allowbreak M_{\mi{final}}) \in {\cal U}_{\mi{SN}}$ be a system net with $N=(P,T,F,l)$.
	$A_{LM} =
	\{ (x,(x,t)) \mid x \in A \ \wedge \ t \in T \ \wedge \ l(t) = x \} \cup
	\{ (\nomove,(x,t)) \mid t \in T \ \wedge \ l(t) = x \} \cup
	\{ (x,\nomove) \mid x \in A \}$ is the set of \emph{legal moves}.
\end{definition}

An alignment is a sequence of legal moves such that after removing all ``$\nomove$'' symbols, the top row corresponds to a trace in the log and the bottom row corresponds to a firing sequence starting in $M_{\mi{init}}$ and ending in $M_{\mi{final}}$. Notice that if $t \notin \mi{dom}(l)$ is an invisible transition, the activation of $t$ is indicated by a ``$\nomove$'' on the log in correspondence of $t$ and the placeholder label $\tau$.
Hence, the middle row corresponds to a visible path when ignoring the $\tau$ steps. Figure~\ref{fig:esconformance} shows a system net with two examples of alignments, $\sigma_1$ of a fitting trace and $\sigma_2$ of a non-fitting trace.

\begin{figure}
	\centering
	\includegraphics[width=.9\linewidth]{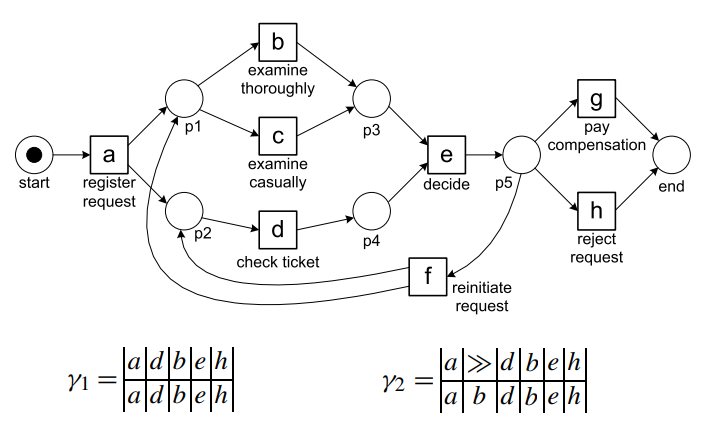}
	\caption{Example of alignments on a system net. The alignment $\gamma_1$ shows that the trace $\langle a,d,b,e,h \rangle$ is perfectly fitting the net. The alignment $\gamma_2$ shows that the trace $\langle a,b,d,b,e,h \rangle$ is misaligned with the net in one point, indicated by ``$\nomove$''. Partially from~\cite{van2016process}.}
	\label{fig:esconformance}
\end{figure}

\begin{definition}[Alignment]\label{def:alignment}
	Let $\sigma \in L$ be a certain trace and $\sigma_T \in \phi_f(\mi{SN})$ a complete firing sequence of system net $\mi{SN}$.
	An \emph{alignment} of $\sigma$ and $\sigma_T$ is a sequence $\gamma \in {A_{LM}}^*$
	such that the projection on the first element (ignoring ``$\nomove$'') yields $\sigma$
	and the projection on the last element (ignoring ``$\nomove$'' and transition labels) yields $\sigma_T$.
\end{definition}

A trace and a model can have several possible alignments. In order to select the most appropriate one, we introduce a function that associates a \emph{cost} to undesired moves - the ones associated with deviations.

\begin{definition}[Cost of Alignment]\label{def:alcosts}
	Cost function $\delta \colon {A_{LM}} \rightarrow \Nat$ assigns costs to legal moves.
	The \emph{cost} of an alignment $\gamma \in {A_{LM}}^*$ is the sum of all costs:
	$\delta(\gamma) = \sum_{(x,y)\in \gamma} \delta(x,y)$.
\end{definition}

Moves where log and model agree have no costs, i.e.,
$\delta(x,(x,t)) = 0$ for all $x\in A$.
Moves on model only have no costs if the transition is invisible, i.e.,
$\delta(\nomove,(\tau,t))=0$ if $l(t) = \tau$.
$\delta(\nomove,(x,t))>0$ is the cost when the model makes an ``$x$ move''
without a corresponding move of the log (assuming $l(t) = x \neq \tau$).
$\delta(x,\nomove) > 0$ is the cost for an ``$x$ move'' only on the log.
In this paper, we often use a standard cost function $\delta_S$
that assigns unit costs:
$\delta_S(x,(x,t)) = 0$,
$\delta_S(\nomove,(\tau,t)) = 0$,
and $\delta_S(\nomove,(x,t)) = \delta_S(x,\nomove) = 1$ for all $x\in A$.

\begin{definition}[Optimal Alignment]\label{def:optalign}
	Let $L \in \bag(\mathcal{T})$ be a certain event log and let $\mi{SN}\in {\cal U}_{\mi{SN}}$ be a system net with $\phi(\mi{SN}) \neq \emptyset$.
	\begin{itemize}
		\item For $\sigma \in L$, we define:
		$\Gamma_{\sigma,\mi{SN}} = \{ \gamma \in  {A_{LM}}^* \mid \exists_{\sigma_T \in \phi_f(\mi{SN})} \  \gamma \ \mathit{is} \ \mathit{an} \ \allowbreak \mathit{alignment} \ \mathit{of} \allowbreak \ \sigma \ \mathit{and} \ \sigma_T \}$.
		
		\item An alignment $\gamma \in \Gamma_{\sigma,SN}$ is \emph{optimal} for trace $\sigma \in L$ and system net $\mi{SN}$
		if for any $\gamma' \in  \Gamma_{\sigma,SN}$: $\delta(\gamma') \geq \delta(\gamma)$.
		
		\item $\lambda_{\mi{SN}} \colon \mathcal{T} \rightarrow {A_{LM}}^* $ is
		a deterministic mapping that assigns any trace $\sigma$ to an optimal alignment, i.e., $\lambda_{\mi{SN}}(\sigma) \in \Gamma_{\sigma,\mi{SN}}$ and $\lambda_{\mi{SN}}(\sigma)$ is optimal.
		
		\item $\mi{costs}(L,\mi{SN},\delta) = \sum_{\sigma \in L} \delta(\lambda_{\mi{SN}}(\sigma))$ are the \emph{misalignment costs} of the whole event log.
	\end{itemize}
	$\sigma \in L$ is a (perfectly) fitting trace for the system net $\mi{SN}$ if and only if $\delta(\lambda_{\mi{SN}}(\sigma)) = 0$. $L$ is a (perfectly) fitting event log for the system net $\mi{SN}$ if and only if $\mi{costs}(L,\mi{SN},\delta) = 0$.
\end{definition}

The technique to compute the optimal alignment~\cite{adriansyah2014aligning} is as follows. Firstly, it creates an \emph{event net}, a sequence-structured system net able to replay only the trace to align. The transitions in the event net have labels corresponding to the activities in the trace. Then, a \emph{product net} should be computed. A product net is the union of the event net and the model together, with synchronous transitions added. These additional transitions are paired with transitions in the event net and in the process model that have the same label. Then, they are connected with arcs from the input places and to the output places of those transitions. The product net is able to represent moves on log, moves on model and synchronous moves by means of firing transitions. In fact, the transitions of the event net correspond to moves on log, the transitions of the process model correspond to moves on model, the added synchronous transitions correspond to synchronous moves. The union of the initial and final markings of the event net and the process model constitute respectively the initial and final marking of the product net, while every complete firing sequence on the product net corresponds to a possible alignment. Lastly, the product net is translated to a state space, and a state space exploration via the $\mathbb{A}^*$ algorithm is performed in order to find the complete firing sequence that yields the lowest cost.

Let us define formally the construction of the event net and the product net:

\begin{definition}[Event Net]
	Let $\sigma \in \mathcal{T}$ be a certain trace. The \emph{event net} $\mi{en}: \mathcal{T} \to \mathcal{U}_{\mi{SN}}$ of $\sigma$ is a system net $\text{en}(\sigma) = (P, T, F, l, M_{init}, M_{final})$ such that:
	\begin{itemize}
		\item $P = \{p_i \mid 1 \leq i \leq |\sigma|+1 \}$,
		\item $T = \{t_i \mid 1 \leq i \leq |\sigma| \}$,
		\item $F = \bigcup_{1 \leq i \leq |\sigma|}\{(p_i, t_i), (t_i, p_{i + 1})\}$
		\item $l \colon T \to \mathcal{U}_A$ such that for all $1 \leq i \leq |\sigma|$, $l(t_i) = \sigma[i]$,
		\item $M_{init} = [p_1]$,
		\item $M_{final} = [p_{|P|}]$.
	\end{itemize}
\end{definition}

Note that the labeling function $l$ of an event net is a total function: no invisible transitions are contained in an event net, since for each event we generate a transition labeled with the corresponding activity label.

\begin{definition}[Product of two Petri Nets~\cite{winskel1987petri}]
	Let $S_1 = (P_1, T_1, F_1, l_1,\allowbreak M_{init_1}, M_{final_1})$ and $S_2 = (P_2, T_2, F_2, l_2, M_{init_2}, M_{final_2})$ be two system nets. The \emph{product net} of $S_1$ and $S_2$ is the system net $S = S_1 \otimes S_2 = (P, T, F, l, M_{init},\allowbreak M_{final})$ such that:
	\begin{itemize}
		\item $P = P_1 \cup P_2$,
		\item $T \subseteq (T_1 \cup \{\nomove\} \times T_2 \cup \{\nomove\})$ such that $T = \{(t_1, \nomove) \mid t_1 \in T_1 \} \cup \{(\nomove, t_2) \mid t_2 \in T_2 \} \cup \{(t_1, t_2) \in (T_1 \times T_2) \mid l_1(t_1) = l_2(t_2) \neq \tau \}$,
		\item $F \subseteq (P \times T) \cup (T \times P)$ such that
		\begin{align*}
			F &= \begin{aligned}[t]
			&\{(p_1, (t_1, \nomove)) \mid p_1 \in P_1 \wedge t_1 \in T_1 \wedge (p_1, t_1) \in F_1 \} \cup
			\\
			&\{((t_1, \nomove), p_1) \mid t_1 \in T_1 \wedge p_1 \in P_1 \wedge (t_1, p_1) \in F_1 \} \cup \\ 
			&\{(p_2, (t_2, \nomove)) \mid p_2 \in P_2 \wedge t_2 \in T_2 \wedge (p_2, t_2) \in F_2 \} \cup \\
			&\{((t_2, \nomove), p_2) \mid t_2 \in T_2 \wedge p_2 \in P_2 \wedge (t_2, p_2) \in F_2 \} \cup \\
			&\{(p_1, (t_1, t_2)) \mid p_1 \in P_1 \wedge (t_1, t_2) \in T \cap (T_1 \times T_2) \wedge (p_1, t_1) \in F_1 \} \cup \\
			&\{(p_2, (t_1, t_2)) \mid p_2 \in P_2 \wedge (t_1, t_2) \in T \cap (T_1 \times T_2) \wedge (p_2, t_2) \in F_2 \} \cup \\
			&\{((t_1, t_2), p_1) \mid p_1 \in P_1 \wedge (t_1, t_2) \in T \cap (T_1 \times T_2) \wedge (t_1, p_1) \in F_1 \} \cup \\
			&\{((t_1, t_2), p_2) \mid p_2 \in P_2 \wedge (t_1, t_2) \in T \cap (T_1 \times T_2) \wedge (t_2, p_2) \in F_2 \}\end{aligned}
		\end{align*}
		\item $l \colon T \to \mathcal{U}_A$ such that for all $(t_1, t_2) \in T$, $l((t_1, t_2)) = l_1(t_1)$ if $t_2 = \nomove$, $l((t1, t2)) = l_2(t_2)$ if $t_1 = \nomove$, and $l((t_1, t_2)) = l_1(t_1)$ otherwise,
		\item $M_{init} = M_{init_1} \uplus M_{init_2}$,
		\item $M_{final} = M_{final_1} \uplus M_{final_2}$.
	\end{itemize}
\end{definition}

\section{Uncertainty in Process Mining}\label{sec:unc}

In this section, we will extend the definitions of event, trace, and event log to the uncertain case. Let us first define the identifiers necessary to express event indeterminacy.

\begin{definition}[Determinate and indeterminate event qualifiers]
	Let $\mathcal{U}_O = \{!, ?\}$, where the ``!'' symbol denotes \emph{determinate events}, and the ``?'' symbol denotes \emph{indeterminate events}.
\end{definition}

For strong uncertainty, attribute values are replaced by a set of possible values. In the case of weak uncertainty, a continuous function $f$ provides the probability density for the combinations of attribute values in the uncertain event. Notice that the total mass of probabilities described by $f$ might be lower than 1: this is so we can aptly represent the case of an indeterminate event.

\begin{definition}[Uncertain events]\label{def:events}
	Let $\mathcal{E}_S = \mathcal{U}_I \times \mathcal{P}_{NE}(\mathcal{U}_C) \times \mathcal{P}_{NE}(\mathcal{U}_A) \times \mathcal{P}_{NE}(\mathcal{U}_T) \times \mathcal{U}_O$ denote the universe of \emph{strongly uncertain events}. $\mathcal{E}_W =\{(e_i, f) \in \mathcal{U}_I \times ((\mathcal{U}_C \times \mathcal{U}_A \times \mathcal{U}_T) \not\to [0,1]) \mid \sum_{(c, a, t) \in dom(f)}f(c, a, t) \leq 1\}$ is the universe of \emph{weakly uncertain events}\footnote{We assume here that $dom(f)$ is finite. It is easy to generalize to the infinite case by employing an integral.}.
\end{definition}

The probability of a weakly uncertain event of having been recorded but not happening in reality is equal to $1 - \sum_{(c, a, t) \in dom(f)}f(c, a, t)$.

Now that the definitions of strongly and weakly uncertain events are given, let us aggregate them in uncertain event logs.

\begin{definition}[Uncertain event logs]
	A \emph{strongly uncertain event log} is a set of events $L_S \subseteq \mathcal{E}_S$ such that every event identifier in $L_S$ is unique. A \emph{weakly uncertain event log} is a set of events $L_W \subseteq \mathcal{E}_W$ such that every event identifier in $L_W$ is unique.
	
	For a strongly uncertain event $e = (e_i, c_s, a_s, t_s, o) \in L_S$ we define the following projection functions: $\pi^{L_S}_c(e) = c_s \in \mathcal{P}_{NE}(\mathcal{U}_C)$, $\pi^{L_S}_a(e) = a_s \in \mathcal{P}_{NE}(\mathcal{U}_A)$, $\pi^{L_S}_t(e) = t_s \in  \mathcal{P}_{NE}(\mathcal{U}_T)$ and $\pi^{L_S}_o(e) = o \in \mathcal{U}_O$.
\end{definition}

A weakly uncertain event log $L_W \subseteq \mathcal{E}_W$ has a corresponding strongly uncertain event log $\overline{L_W} = L_S \subseteq \mathcal{E}_S$ such that
\begin{align*}
	L_S &= \begin{aligned}[t]&
	\{(e_i, c_s, a_s, t_s, o) \in \mathcal{E}_S \mid \exists_{({e_i}', f) \in L_W} e_i = {e_i}' \wedge \\ &
	c_s = \{c \in \mathcal{U}_C \mid \exists_{a, t} \ (c, a, t) \in dom(f) \wedge f(c, a, t) > 0\} \wedge \\ &
	a_s = \{a \in \mathcal{U}_A \mid \exists_{c, t} \ (c, a, t) \in dom(f) \wedge f(c, a, t) > 0\} \wedge \\ &
	t_s = \{t \in \mathcal{U}_T \mid \exists_{c, a} \ (c, a, t) \in dom(f) \wedge f(c, a, t) > 0\} \wedge \\ &
	(o = \:! \Leftrightarrow \sum_{(c, a, t) \in dom(f)}f(c, a, t) = 1) \wedge \\ &
	(o = \:? \Leftrightarrow \sum_{(c, a, t) \in dom(f)}f(c, a, t) < 1) \}.\end{aligned}
\end{align*}

Notice that representing the density of probability for combinations of values of case ID, time and activity with a single function $f$ is an approximation that assumes probabilistic independence between event attributes.

\begin{definition}[Realization of an event log]
	$L_C \subseteq \mathcal{E}_C$ is a \emph{realization} of $L_S \subseteq \mathcal{E}_S$ if and only if:
	\begin{itemize}
		\item For all $(e_i, c, a, t) \in L_C$ there is a distinct $({e_i}', c_s, a_s, t_s, o) \in L_S$ such that $e_i = {e_i}'$, $c \in c_s$, $a \in a_s$ and $t \in t_s$;
		\item For all $(e_i, c_s, a_s, t_s, o) \in L_S$ with $o = \:!$ there is a distinct $({e_i}', c, a, t) \in L_C$ such that $e_i = {e_i}'$, $c \in c_s$, $a \in a_s$ and $t \in t_s$.
	\end{itemize}
	$\mathcal{R}_L(L_S)$ is the set of all such realizations of the log $L_S$.
\end{definition}

Note that these definitions allow us to transform a weakly uncertain log into a strongly uncertain one, and a strongly uncertain one in a set of certain logs.

In this paper, we focus on three types of uncertainty:
\begin{itemize}
	\item Strong uncertainty on the activity;
	\item Strong uncertainty on the timestamp;
	\item Strong uncertainty on indeterminate events.
\end{itemize}
All three can happen concurrently. Following the taxonomy presented in Section~\ref{sec:taxonomy}, this setting corresponds to a $\text{[E, A, T]}_\mathbb{S}$-type log.
It is worth noting that the specific case of uncertainty on the case ID causes a problem; since an event can have many possible case IDs, it can belong to different traces. In data format where the events are already aggregated into traces, such as the very common XES standard, this means that the information related to a trace can be \emph{non-local} to the trace itself, but can be stored in some other points of the log. We will focus on the problem of uncertainty on the case ID attribute in future work.

Firstly, we will lay down some simplified notation in order to model the problem at hand in a more compact way.

\begin{definition}[Simple uncertain events, traces and logs]\label{def:unc}	
	Let $e_i \in \mathcal{U}_I$, $a_s \in \mathcal{P}_{NE}(\mathcal{U}_A)$, $t_{min} \in \mathcal{U}_T$, $t_{max} \in \mathcal{U}_T$ and $o \in \mathcal{U}_O$ such that $t_{min} < t_{max}$. $e_U^S = (e_i, a_s, t_{min}, t_{max}, o)$ is a \emph{simple uncertain event}. Let us denote with $\mathcal{E}_U^S \subseteq \mathcal{U}_I \times \mathcal{P}_{NE}(\mathcal{U}_A) \times \mathcal{U}_T \times \mathcal{U}_T \times \mathcal{U}_O$ the universe of all simple uncertain events. $\sigma_U \subseteq \mathcal{E}_U^S$ is a \emph{simple uncertain trace} if all the event identifiers in $\sigma_U$ are unique. $\mathcal{T}_U$ denotes the universe of simple uncertain traces.  $L_U \in \mathcal{P}(\mathcal{T}_U)$ is a \emph{simple uncertain log} if all the event identifiers in $L_U$ are unique. For $\sigma_U \in L_U$ and $e_U^S = (e_i, a_s, t_{min}, t_{max}, o) \in \sigma_U$ we define the following projection functions: $\pi^{L_U}_a(e_U^S) = a_s \in \mathcal{P}_{NE}(\mathcal{U}_A)$, $\pi^{L_U}_{t_{min}}(e_U^S) = t_{min} \in \mathcal{U}_T$, $\pi^{L_U}_{t_{max}}(e_U^S) = t_{max} \in \mathcal{U}_T$ and $\pi^{L_U}_o(e_U^S) = o \in \mathcal{U}_O$.
\end{definition}

In a simple uncertain event $e_U^S = (e_i, a_s, t_{min}, t_{max}, o)$, the true activity label of the event is one of the labels contained in the set $a_s$, the true timestamp is one of the values contained in the closed interval $[t_{min}, t_{max}]$, while the indeterminacy symbol $o$ indicates whether the event has certainly occurred, or if it is possible that it did not occur even though it has been recorded in an event log.

Simple uncertain events are best illustrated with a running example. Let us consider the following process instance, a simplified version of anomalies that are actually occurring in processes of the healthcare domain. An elderly patient enrolls in a clinical trial for an experimental treatment against myeloproliferative neoplasms, a class of blood cancers. The enrollment in this trial includes a lab exam and a visit with a specialist; then, the treatment can begin. The lab exam, performed on the 8th of July, finds a low level of platelets in the blood of the patient, a condition known as thrombocytopenia (TP). At the visit, on the 10th of May, the patient self-reports an episode of night sweats on the night of the 5th of July, prior the lab exam: the medic notes this, but also hypothesized that it might not be a symptom, since it can be caused not by the condition but by external factors (such as very warm weather). The medic also reads the medical records of the patient and sees that, shortly prior to the lab exam, the patient was undergoing a heparine treatment (a blood-thinning medication) to prevent blood clots. The thrombocytopenia found with the lab exam can then be primary (caused by the blood cancer) or secondary (caused by other factors, such as a drug). Finally, the medic finds an enlargement of the spleen in the patient (splenomegaly). It is unclear when this condition has developed: it might have appeared at any moment prior to that point. The medic decides to admit the patient in the clinical trial, starting 12th of July. These events are collected and recorded in the trace shown in Table~\ref{table:uncertaintrace} in the information system of the hospital.  For readability, the timestamp field only indicates the day of the month. This trace includes all types of uncertainty contained in a $\text{[E, A, T]}_\mathbb{S}$-type log, the setting we are considering for the application of conformance checking.

\begin{table}[]
	\caption{The uncertain trace of an instance of healthcare process used as a running example. For sake of clarity, we have further simplified the notation in the timestamps column, by showing only the day of the month.}
	\label{table:uncertaintrace}
	\centering
	\begin{tabular}{ccccc}
		\textbf{Case ID}        & \textbf{Event ID} & \textbf{Timestamp}                                                                                                     & \textbf{Activity}             & \multicolumn{1}{l}{\textbf{Indet. event}} \\ \hline
		\multicolumn{1}{|c|}{ID192} & \multicolumn{1}{c|}{$e_1$} 
		& \multicolumn{1}{c|}{5}                                                                         & \multicolumn{1}{c|}{\emph{NightSweats}}        & \multicolumn{1}{c|}{?}                    \\ \hline
		\multicolumn{1}{|c|}{ID192}& \multicolumn{1}{c|}{$e_2$} & \multicolumn{1}{c|}{8}                                                                         & \multicolumn{1}{c|}{\{\emph{PrTP}, \emph{SecTP}\}} & \multicolumn{1}{c|}{!}                    \\ \hline
		\multicolumn{1}{|c|}{ID192}& \multicolumn{1}{c|}{$e_3$} & \multicolumn{1}{c|}{[4, 10]}                                                                         & \multicolumn{1}{c|}{\emph{Splenomeg}} & \multicolumn{1}{c|}{!}                    \\ \hline
		\multicolumn{1}{|c|}{ID192}& \multicolumn{1}{c|}{$e_4$} & \multicolumn{1}{c|}{12}                                                                         & \multicolumn{1}{c|}{\emph{Adm}}        & \multicolumn{1}{c|}{!}                    \\ \hline
	\end{tabular}
\end{table}

In the notation of Definition~\ref{def:unc}, the trace
$\sigma_U$ in Table~\ref{table:uncertaintrace} is denoted as:
\begin{align*}
\sigma_U = \{(e_1, \{\textit{NightSweats}\}, 5, 5, ?), (e_2, \{\textit{PrTP}, \textit{SecTP}\}, 8, 8, !),\\ (e_3, \{\textit{Splenomeg}\}, 4, 10, !), (e_4, \{\textit{Adm}\}, 12, 12, !)\}.
\end{align*}

We can also draw the time diagram of this example of uncertain trace, which can be seen in Figure~\ref{fig:gantt_running}.

\begin{figure}[h!]
	\centering
	\includegraphics[width=.9\columnwidth]{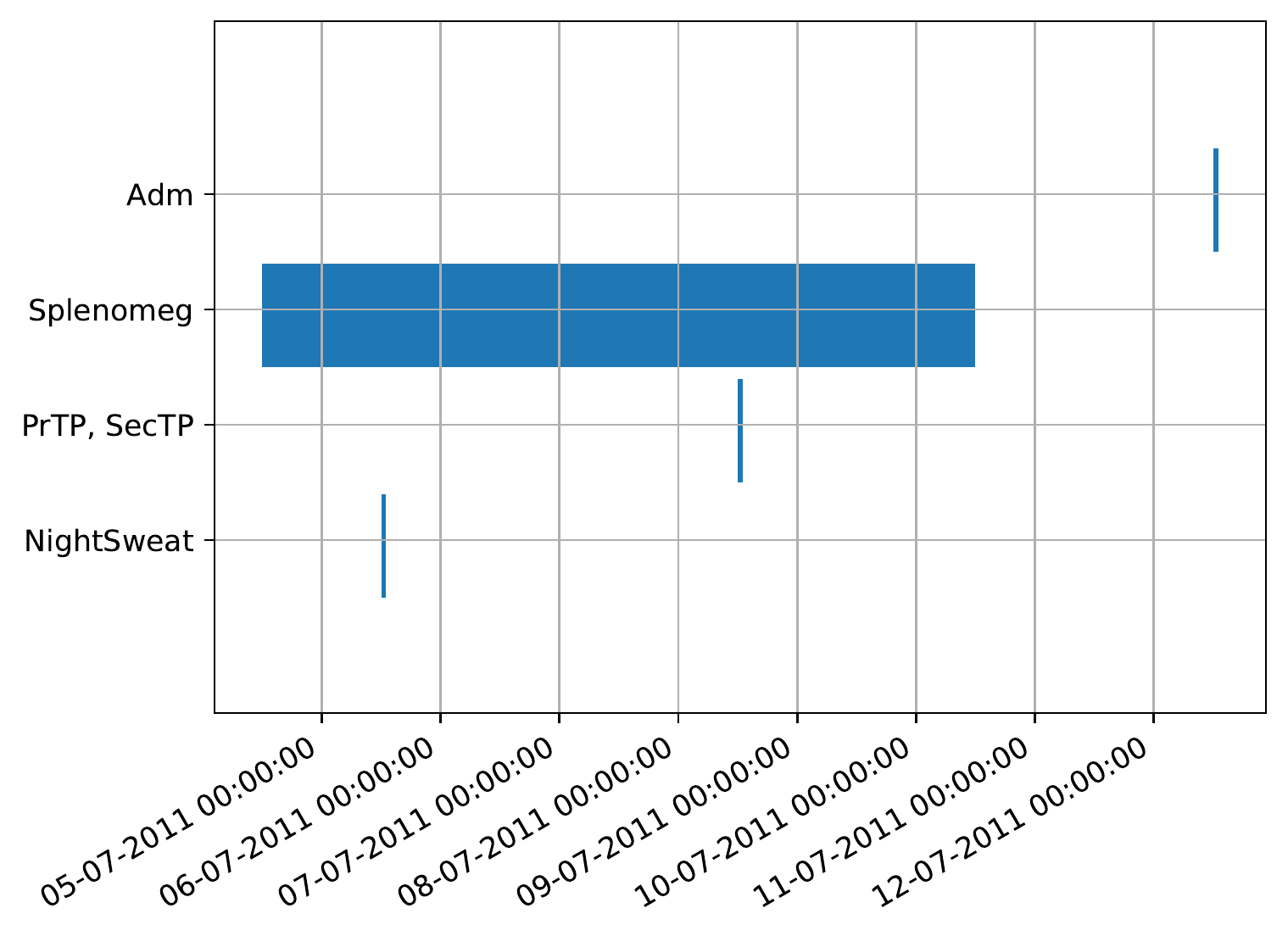}
	\caption{Time diagram of the trace in Table~\ref{table:uncertaintrace}.}
	\label{fig:gantt_running}
\end{figure}

In the remainder of the paper, when defining simple uncertain traces and events, we always assume that these belong to a corresponding simple uncertain log. Thus, for simplicity, we will omit the qualifier ``$L_U$'' when denoting the corresponding projection functions.

These simplified traces and logs can be related to the more general framework described in the previous section through the following transformation: let $L_S \subseteq \mathcal{E}_S$ be a strongly uncertain log and let $g \colon \mathcal{U}_I \not\to \mathcal{U}_C$ be a function mapping event identifiers onto cases such that $dom(g) = \{e_i \mid (e_i, c_s, a_s, t_s, u) \in L_S\}$ and for all $(e_i, c_s, a_s, t_s, u) \in L_S$, $g(e_i) \in c_s$. Thus, for $c \in rng(g)$, $g^{-1}(c) = \{e_i \in \mathcal{U}_I \mid g(e_i) = c\}$. The simple uncertain event log defined by $g$ on $L_S$ is given as $L_U = \{ \{(e_i, \pi^{L_S}_a(e), min(\pi^{L_S}_t(e)), max(\pi^{L_S}_t(e)), \pi^{L_S}_o(e)) \mid e_i \in g^{-1}(c) \wedge \pi^{L_S}_i(e) = e_i\} \mid c \in rng(g) \}$.

In order to more easily work with timestamps in simple uncertain events, let us frame their time relationship as a strict partial order.

\begin{definition}[Strict partial order over simple uncertain events]\label{def:ord}
	Let $e, e' \in \mathcal{E}_U^S$ be two simple uncertain events. $\prec_\mathcal{E}$ is a strict partial order defined on the universe of strongly uncertain events $\mathcal{E}_U^S$ as:
	\[
	e \prec_\mathcal{E} e' \Leftrightarrow \pi_{t_{max}}(e) < \pi_{t_{min}}(e')
	\]
\end{definition}

\begin{proposition}[$\prec_\mathcal{E}$ is a strict partial order]\label{th:ord}
\end{proposition}
\begin{proof}
	All properties characterizing strict partial orders are fulfilled by $\prec_\mathcal{E}$. For all $e, e', e'' \in \mathcal{E}_U^S$ we have:
	\begin{itemize}
		\item Irreflexivity: this property is always verified, since $\pi_{t_{max}}(e) < \pi_{t_{min}}(e)$ is false (see Definition~\ref{def:events}).
		\item Transitivity: since $\pi_{t_{max}}(e) < \pi_{t_{min}}(e') \leq \pi_{t_{max}}(e') < \pi_{t_{min}}(e'')$ and $\mathcal{U}_T$ is totally ordered, we have that $\pi_{t_{max}}(e) < \pi_{t_{min}}(e'')$ and this property is always verified.
	\end{itemize}
\end{proof}

\begin{lemma}[Uncomparable events share possible timestamp values]\label{lem:uncomp}
	Let $e, e' \in \mathcal{E}_U^S$ be two strongly uncertain events. $e$ and $e'$ are uncomparable with respect to the strict partial order $\prec_\mathcal{E}$ (i.e., neither $e \prec_\mathcal{E} e'$ nor $e' \prec_\mathcal{E} e$ are true) if and only if $e$ and $e'$ share some possible values of their timestamp.
\end{lemma}
\begin{proof}
	\leavevmode \\
	$(\Rightarrow)$ From Definition~\ref{def:ord}, it follows that two events $e, e' \in \mathcal{E}_U^S$ are comparable if and only if either $\pi_{t_{max}}(e) < \pi_{t_{min}}(e')$ or $\pi_{t_{max}}(e') < \pi_{t_{min}}(e)$. If both are false, then $\pi_{t_{min}}(e') \leq \pi_{t_{max}}(e)$ and $\pi_{t_{min}}(e) \leq \pi_{t_{max}}(e')$. If we assume that $\pi_{t_{min}}(e) \leq \pi_{t_{min}}(e')$ then $\pi_{t_{min}}(e) \leq \pi_{t_{min}}(e') \leq \pi_{t_{max}}(e)$, while if $\pi_{t_{min}}(e) > \pi_{t_{min}}(e')$ then $\pi_{t_{min}}(e') < \pi_{t_{min}}(e) \leq \pi_{t_{max}}(e')$. In both cases, there are values common to both uncertain timestamps.
	
	\leavevmode \\
	$(\Leftarrow)$ If the two events share timestamp values, it follows that at least one of the extremes of one event is encompassed by the extremes of the other. Assume that $e$ encompasses at least one of the extremes of $e'$ (the other case is symmetric): then either $\pi_{t_{min}}(e) \leq \pi_{t_{min}}(e') \leq \pi_{t_{max}}(e)$ or $\pi_{t_{min}}(e) \leq \pi_{t_{max}}(e') \leq \pi_{t_{max}}(e)$. In the first case, considering that $\mathcal{U}_T$ is totally ordered and that $\pi_{t_{min}}(e') \leq \pi_{t_{max}}(e')$, we have that both $\pi_{t_{min}}(e') \leq \pi_{t_{max}}(e)$ and $\pi_{t_{min}}(e) \leq \pi_{t_{max}}(e')$ are true, and $e$ and $e'$ are uncomparable. The second case is proved analogously.
\end{proof}

\begin{definition}[Realizations of simple uncertain traces]\label{def:real}
	Let $\sigma_U \in \mathcal{T}_U$ be a simple uncertain trace. An \emph{order-realization} $\sigma_O \in \mathcal{S}_{\sigma_U}$ is a permutation of the events in $\sigma_U$ such that for all $1 \leq i < j \leq |\sigma_O|$ we have that $\sigma_O[j] \nprec_\mathcal{E} \sigma_O[i]$, i.e., $\sigma_O$ is a correct evaluation order for $\sigma_U$ over $\prec_\mathcal{E}$, and the (total) order in which events are sorted in $\sigma_O$ is a linear extension of the strict partial order $\prec_\mathcal{E}$. We denote with $\mathcal{R}_O(\sigma_U)$ the set of all such order-realizations of the trace $\sigma_U$.
	
	%
	
	Given an order-realization $\sigma_O \in \mathcal{R}_O(\sigma_U)$, the sequence $\sigma = \langle a_1, a_2, \dots, a_n \rangle \in {\mathcal{U}_A}^*$ is a \emph{realization} of $\sigma_O$ if there exists a total function $f \colon \{1, 2, \dots, n\} \rightarrow \sigma_O$ such that:
	\begin{itemize}
		\item For all $1 \leq i \leq n$, $a_i \in \pi_a(f(i))$,
		\item $\langle f(1), f(2), \dots, f(n) \rangle$ is a subsequence of $\sigma_O$,
		\item For all $e \in \sigma_O$ with $\pi_o(\sigma_O) = \:!$ there exists $1 \leq i \leq n$ such that $f(i) = e$.
	\end{itemize}
\end{definition}

We denote with $\mathcal{R}'(\sigma_O) \subseteq {\mathcal{U}_A}^*$ the set of all such realizations of the order-realization $\sigma_O$. We denote with $\mathcal{R}(\sigma_U) \subseteq {\mathcal{U}_A}^*$ the union of the realizations obtainable from all the order-realizations of $\sigma_U$: $\mathcal{R}(\sigma_U) = \bigcup_{\sigma_O \in \mathcal{R}_O(\sigma_U)} \mathcal{R}'(\sigma_O)$.

Let us see some examples of realizations of uncertain traces. Let $\sigma_U$ be the uncertain trace shown in Table~\ref{table:uncertaintrace}. We then have that 
$\sigma_U$ has three order-realizations:

$$
\mathcal{R}_O(\sigma_U) = \{\langle e_3, e_1, e_2, e_4 \rangle, \langle e_1, e_3, e_2, e_4 \rangle, \langle e_1, e_2, e_3, e_4 \rangle\}
$$

We can then compute the realizations of one of the order-realizations of $\sigma_U$:

\begin{align*}
\mathcal{R}'(\langle e_1, e_2, e_3, e_4 \rangle) = \{\langle \textit{NightSweats}, \textit{PrTP}, \textit{Splenomeg}, \textit{Adm} \rangle,\\ \langle \textit{NightSweats}, \textit{SecTP}, \textit{Splenomeg}, \textit{Adm} \rangle,\\ \langle \textit{PrTP}, \textit{Splenomeg}, \textit{Adm} \rangle,\\ \langle \textit{SecTP}, \textit{Splenomeg}, \textit{Adm} \rangle\}
\end{align*}

Simple uncertain traces and logs carry less information than their certain counterparts. Nevertheless, it is possible to extend existing process mining algorithms to extract the information in a simple uncertain log to design a process model that describes its possible behavior, or verify that it conforms to a given normative model.

\section{Conformance Checking on Uncertain Event Data}\label{sec:conformance}

Depending on the possible values for $a_s$, $t_{min}$, $t_{max}$, and $u$ there are multiple possible realizations of a trace. This means that, given a model, a simple uncertain trace could be fitting for certain realizations, but non-fitting for others. The question we are interested in answering is: given a simple uncertain trace and a Petri net process model, is it possible to find an \emph{upper and lower bound} for the conformance score? Usually we are interested in the optimal alignments (the ones with the minimal cost). However, we are now interested in the minimum and maximum cost of alignments in the realization set of a simple uncertain trace.

\begin{definition}[Upper and Lower Bound on Alignment Cost for a Trace]
	Let $\sigma_U \in \mathcal{T}_U$ be a simple uncertain trace, and let $\mi{SN} \in \mathcal{U}_{\mi{SN}}$ be a system net. The \emph{upper bound for the alignment cost} is a function $\delta_{max} \colon \mathcal{T}_U \to \mathbb{N}$ such that $\delta_{max}(\sigma_U) = \max_{\sigma \in \mathcal{R}(\sigma_U)} \delta(\lambda_{\mi{SN}}(\sigma))$. The \emph{lower bound for the alignment cost} is a function $\delta_{min} \colon \mathcal{T}_U \to \mathbb{N}$ such that $\delta_{min}(\sigma_U) = \min_{\sigma \in \mathcal{R}(\sigma_U)} \delta(\lambda_{\mi{SN}}(\sigma))$.
\end{definition}

A simple way to compute the upper and lower bounds for the cost of any uncertain trace is using a brute-force approach: enumerating the possible realizations of the trace, then searching for the costs of optimal alignments for all the realizations, and picking the minimum and maximum as bounds. We now present a technique which improves the performance of calculating the lower bound for conformance cost with respect to a brute-force method.

We will produce a version of the event net that embeds the possible behaviors of the uncertain trace. We define a \emph{behavior net}, a Petri net that can replay all and only the realizations of an uncertain trace. As an intermediate step in order to obtain such a Petri net, we first build the \emph{behavior graph}, a dependency graph representing the uncertain trace. This graph contains a vertex for each uncertain event in the trace and contains an edge between two vertices if the corresponding uncertain events happen one directly after the other in at least one realization of the uncertain trace.


\begin{definition}[Behavior Graph]\label{def:bg}
	Let $\sigma_U \in \mathcal{T}_U$ be a simple uncertain trace. A behavior graph $\beta \colon \mathcal{T}_U \to \mathcal{U}_G$ is the transitive reduction of a directed graph $\rho(G)$, where $G = (V, E) \in \mathcal{U}_G$ is defined as:
	\begin{itemize}
		\item $V = \{e \in \sigma_U \}$,
		\item $E = \{(v, w) \mid v, w \in V \wedge v \prec_\mathcal{E} w\}$.
	\end{itemize}
\end{definition}

The behavior graph provides a structured representation of the uncertainty on the timestamp: when a specific vertex has two or more outbound edges, the events corresponding to the destination vertices can occur in any order, concurrently with each other. We can see the result on the example trace in Figures~\ref{fig:graphcomp} and~\ref{fig:graphred}.

%
%
%
%
%
%

\begin{figure}
	\centering
	\begin{minipage}[t]{0.48\textwidth}
		\centering
		\begin{tikzpicture}[->, node distance=3.5cm, nodes={draw, ellipse, scale=.65}]
		
		\node[dashed]	(A)	[label=below:$e_1$]								{$\text{NightSweats}$};
		\node			(B)	[right of=A, label=below:$e_2$]					{$\{\text{PrTP, SecTP}\}$};
		\node			(C)	[below of=B, yshift=1.5cm, label=below:$e_3$]	{$\text{Splenomeg}$};
		\node			(D)	[right of=C, yshift=1cm, label=below:$e_4$]		{$\text{Adm}$};
		
		\path
		(A) edge (B)
		edge [bend left] (D)
		(B) edge (D)
		(C) edge (D);
		\end{tikzpicture}
		\caption{The graph of the trace in Table~\ref{table:uncertaintrace} before applying the transitive reduction. All the nodes in the graph are pairwise connected based on precedence relationships; pairs of nodes for which the order is unknown are not connected. The dashed node represents an indeterminate event.}
		\label{fig:graphcomp}
	\end{minipage}\hfill
	\begin{minipage}[t]{0.48\textwidth}
		\centering
		\begin{tikzpicture}[->, node distance=3.5cm, nodes={draw, ellipse, scale=.65}]
		
		\node[dashed]	(A)	[label=below:$e_1$]								{$\text{NightSweats}$};
		\node			(B)	[right of=A, label=below:$e_2$]					{$\{\text{PrTP, SecTP}\}$};
		\node			(C)	[below of=B, yshift=1.5cm, label=below:$e_3$]	{$\text{Splenomeg}$};
		\node			(D)	[right of=C, yshift=1cm, label=below:$e_4$]		{$\text{Adm}$};
		
		\path
		(A) edge (B)
		(B) edge (D)
		(C) edge (D);
		\end{tikzpicture}
		\caption{The behavior graph of the trace in Table~\ref{table:uncertaintrace}. The transitive reduction removed the arc between $e_1$ and $e_4$, since they are reachable through $e_2$. This graph has a minimal number of arcs while conserving the same reachability relationship between nodes.}
		\label{fig:graphred}
	\end{minipage}
\end{figure}

\begin{theorem}[Correctness of behavior graphs]~\label{th:bg-correct}
	Let $\sigma_U \in \mathcal{T}_U$ be a simple uncertain trace and $bg(\sigma_U) = (V,E)$ be its behavior graph. The behavior graph $bg(\sigma_U)$ is acyclic; additionally, the set of all topological sortings of the behavior graph corresponds to the set of order-realizations of $\sigma_U$: $\mathcal{O}_{bg(\sigma_U)} = \mathcal{R}_O(\sigma_U)$.
\end{theorem}
\begin{proof}
	From Proposition~\ref{th:ord} we know that $\prec_\mathcal{E}$ is a strict partial order. Let $p = \langle p_1, p_2, \dots, p_m \rangle \in P_{bg}$ be a path in the behavior graph: if $p$ was a cycle, that means that according to Definition~\ref{def:bg} we have $p_1 \prec_\mathcal{E} p_2 \prec_\mathcal{E} \dots \prec_\mathcal{E} p_m \prec_\mathcal{E} p_1$. Since $\prec_\mathcal{E}$ is transitive, we have that $p_1 \prec_\mathcal{E} p_m$ and $p_m \prec_\mathcal{E} p_1$, which would violate the antisymmetry property in Definition~\ref{def:st-par-ord} and would contradict Proposition~\ref{th:ord}. Thus the behavior graph is necessarily acyclic.
	
	The result $\mathcal{O}_{bg(\sigma_U)} = \mathcal{R}_O(\sigma_U)$ immediately follows from Definitions~\ref{def:topsort},~\ref{def:real} and~\ref{def:bg}, and from Proposition~\ref{th:ord}.
\end{proof}

\begin{lemma}[Semantics of behavior graphs]
	Events connected by paths in a given behavior graph have a precedence relationship; events not connected by any paths share possible values for their timestamps and thus might have happened in any order.
\end{lemma}
\begin{proof}
	Immediately follows from Proposition~\ref{th:ord}, Theorem~\ref{th:bg-correct}, and from Lemma~\ref{lem:uncomp}.
\end{proof}

We then obtain a \emph{behavior net} by replacing every vertex in the behavior graph with one or more transitions in an XOR configuration, each representing an activity contained in the $\pi_a$ set of the corresponding uncertain event. Every edge of the behavior graph becomes a place in the behavior net, connected from and to the transitions corresponding to, respectively, its source and target nodes in the graph.

\begin{definition}[Behavior Net]\label{def:bn}
	Let $\sigma_U \in \mathcal{T}_U$ be a simple uncertain trace, and let $bg(\sigma_U) = (V, E)$ be the corresponding behavior graph. A \emph{behavior net} $bn \colon \mathcal{T}_U \to \mathcal{U}_{\mi{SN}}$ is a system net $bn(\sigma_U) = (P, T, F, l, M_{init}, M_{final})$ such that:
	\begin{itemize}
		\item $P = E \cup \\
		\{ (\textsc{start}, v) \mid v \in V \wedge \nexists_{v' \in V}(v', v) \in E \} \cup \\
		\{ (v, \textsc{end}) \mid v \in V \wedge \nexists_{v' \in V}(v, v') \in E \}$,
		\item $T = \{(v, a) \mid v \in V \wedge a \in \pi_a(v)\} \cup \{(v, \tau) \mid v \in V \wedge \pi_o(v) = \:?\}$,
		\item $F = \{((\textsc{start}, v_1), (v_2, a)) \in E \times T \mid v_1 = v_2 \} \cup \\
		\{((v_1, a),(v_2, w)) \in T \times E \mid v_1 = v_2 \} \cup \\
		\{((v, w_1),(w_2, a)) \in E \times T \mid w_1 = w_2 \} \cup \\
		\{((v_1, a), (v_2, \textsc{end}) \in T \times E \mid v_1 = v_2 \}$,
		\item $l = \{((v, a), a) \mid (v, a) \in T \wedge a \neq \tau\}$,
		\item $M_{\text{init}} = [(\textsc{start}, v) \in P \mid v \in V]$,
		\item $M_{\text{final}} = [(v, \textsc{end}) \in P \mid v \in V]$.
	\end{itemize}
\end{definition}

\begin{figure}
	\centering
	\begin{tikzpicture}[node distance=.5cm and .9cm, >=stealth']
	
	\tikzstyle{place} = [circle,draw,thick,minimum size=6mm]
	\tikzstyle{transition} = [rectangle,draw,thick,minimum size=4mm]
	\tikzstyle{invisible} = [transition, fill=black]
	\tikzstyle{finaltoken} = [token, fill=black!30]
	
	\node [place,tokens=1] (p1) [label=above:{\scriptsize $(\textsc{start}, e_1)$}] {};
	
	\node [transition] (t1) [above right= of p1, label=above:{\scriptsize $(e_1, NightSweats)$}] {NightSweats};
	\draw [->] (p1) to (t1.west);
	
	\node [invisible] (t2) [below right= of p1, label=above:{\scriptsize $(e_1, \tau)$}] {NightSweats};
	\draw [->] (p1) to (t2.west);
	
	\node [place] (p2) [below right= of t1, label=above:{\scriptsize $(e_1, e_2)$}] {};
	\draw [->] (t1.east) to (p2);
	\draw [->] (t2.east) to (p2);
	
	\node [transition] (t3) [above right= of p2, label=above:{\scriptsize $(e_2, PrTP)$}] {PrTP};
	\draw [->] (p2) to (t3.west);
	
	\node [transition] (t4) [below right= of p2, label=above:{\scriptsize $(e_2, SecTP)$}] {SecTP};
	\draw [->] (p2) to (t4.west);
	
	\node [place] (p3) [below right= of t3, label=above:{\scriptsize $(e_2, e_4)$}] {};
	\draw [->] (t3.east) to (p3);
	\draw [->] (t4.east) to (p3);
	
	\node [place,tokens=1] (p4) [below left= of t2, label=above:{\scriptsize $(\textsc{start}, e_3)$}] {};
	
	\node [place] (p5) [below right= of t4, label=above:{\scriptsize $(e_3, e_4)$}] {};
	
	\node [transition] (t5) at ($(p4)!0.5!(p5)$) [label=above:{\scriptsize $(e_3, Splenomeg)$}] {Splenomeg};
	\draw [->] (p4) to (t5);
	\draw [->] (t5) to (p5);
	
	\node [transition] (t6) [above right= of p5, label=above:{\scriptsize $(e_4, Adm)$}] {Adm};
	\draw [->] (p3) to (t6.north west);
	\draw [->] (p5) to (t6.south west);
	
	\node [place] (p6) [right= of t6, label=above:{\scriptsize $(e_4, \textsc{end})$}] {};
	\draw [->] (t6) to (p6);
	\node [finaltoken] at (p6) {};
	\end{tikzpicture}
	\caption{The behavior net corresponding to the uncertain trace in Table~\ref{table:uncertaintrace}. The labels show the objects involved in the construction of Definition~\ref{def:bn}. The initial marking is displayed; the gray ``token slot'' represents the final marking.}
	\label{fig:behnet}
\end{figure}
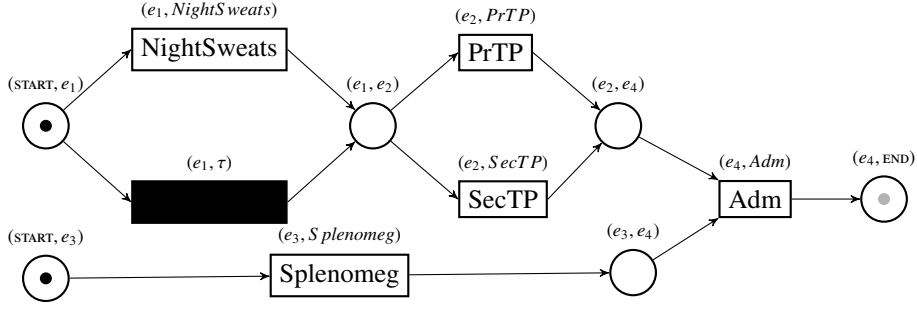

In Figure~\ref{fig:behnet}, we can see the behavior net corresponding to the uncertain trace in Table~\ref{table:uncertaintrace}. It is important to note that every set of edges in the behavior graph with the same source vertex generates an AND split in the behavior net, and a set of edges with the same destination vertex generates an AND join. At the same time, the transitions whose labels correspond to different possible activities in an uncertain event will appear in an XOR construct inside the behavior net.

Thus, in the behavior net, every set of events which timestamps share some possible values will be represented by transitions inside an AND construct, and will then be able to execute in any order allowed by their uncertain timestamp attributes. In the same fashion, an event with uncertainty on the activity will be represented by a number of transitions in an XOR construct. This allows replaying any possible choice for the activity attribute. It follows that, by construction, for a certain simple uncertain trace $\sigma_U$ we have that $\phi(bn(\sigma_U)) = \mathcal{R}(\sigma_U)$.

We can use the behavior net of an uncertain trace $\sigma_U$ in lieu of the event net to compute alignments with a model $\mi{SN} \in \mathcal{U}_{\mi{SN}}$; the search algorithm returns an optimal alignment, a sequence of moves $(x, (y, t))$ with $x \in \mathcal{U}_A$, $y \in \mathcal{U}_A$ and $t$ transition of the model $\mi{SN}$. After removing all ``$\nomove$'' symbols, the sequence of first elements of the moves will describe a complete firing sequence $\sigma_{bn}$ of the behavior net. Since $\sigma_{bn}$ is complete, $\sigma_{bn} \in \phi(bn(\sigma_U))$ and, thus, $\sigma_{bn} \in \mathcal{R}(\sigma_U)$. It follows that $\sigma_{bn}$ is a realization of $\sigma_U$, and the search algorithm ensures that $\sigma_{bn}$ is a realization with optimal conformance cost for the model $\mi{SN}$: $\delta(\lambda_{\mi{SN}}(\sigma_{bn})) = \min_{\sigma \in \mathcal{R}(\sigma_U)} \lambda_{\mi{SN}}(\sigma) = \delta_{min}(\sigma_U)$.

\begin{theorem}[Correctness of behavior nets]\label{th:bn-correct}
	Let $\sigma_U \in \mathcal{T}_U$ be a simple uncertain trace and let $bg(\sigma_U) = (V, E)$ be its behavior graph. The corresponding behavior net $bn(\sigma_U) = (P, T, F, l, M_{init},\allowbreak M_{final})$ can replay all and only the realizations of $\sigma_U$: $\phi(bn(\sigma_U)) = \mathcal{R}(\sigma_U)$.
\end{theorem}
\begin{proof}
	Let $(v, v') \in E$ be an edge of the behavior graph, which also defines a place in the behavior net: $(v, v') = p_{v, v'} \in P$. Let us denote with $\mathbb{T}_v$ the set of transitions in the behavior net generated from the vertex $v$: $\mathbb{T}_v = \{(v', a) \in T \mid v' = v \}$.
	
	\leavevmode \\
	$(\subseteq)$ Let $\sigma = \langle a_1, a_2, \dots, a_n \rangle \in \phi(bn(\sigma_U))$ be any certain trace accepted by $bn(\sigma_U)$. Let $\sigma_T = \langle t_1, t_2, \dots, t_n \rangle \in \phi_f(bn(\sigma_U))$ be a complete firing sequence of $bn(\sigma_U)$ yielding $\sigma$, i.e., $l(\sigma_T)\tproj_{\mathcal{U}_A} = \sigma$. Let $\langle v_1, v_2, \dots, v_n \rangle$ be a sequence of vertices in $bg(\sigma_U)$ such that $t_1 = (v_1, a_1), t_2 = (v_2, a_2), \dots,\allowbreak t_n = (v_n, a_n)$ and $t_1 \in \mathbb{T}_{v_1}, t_2 \in \mathbb{T}_{v_2}, \dots, t_n \in \mathbb{T}_{v_n}$. Let $\mathcal{V}$ be the set of all such sequences; by the flow relation in Definition~\ref{def:bn} there must exist a sequence $\sigma_O = \langle v_1, v_2, \dots, v_n \rangle \in \mathcal{V}$ such that $((v_1, a_1),(v_1, v_2)) \in F, ((v_1, v_2),(v_2, a_2)) \in F, ((v_2, a_2),(v_2, v_3)) \in F, ((v_2, v_3),(v_3, a_3)) \in F, \dots, ((v_{n-1}, a_{n-1}),(v_{n-1}, v_n)) \in F, ((v_{n-1}, v_n),(v_n, a_n))\allowbreak \in F$. This implies that $(v_1, v_2) \in E, (v_2, v_3)\allowbreak \in E, \dots, (v_{n-1}, v_n) \in E$. From Definition~\ref{def:bg} we then have that $v_1 \nsucc_\mathcal{E} v_2 \nsucc_\mathcal{E} \dots \nsucc_\mathcal{E} v_n$. Furthermore, since there exist a $\mathbb{T}_v$ for all $v \in V$ and for all $1 \leq i \leq n$ exactly one transition $t_i \in \mathbb{T}_{v_i}$ has to fire to complete the firing sequence, we have that for all $v \in V$, $v \in \sigma_O$ and is unique. Thus, $\sigma_O \in \mathcal{S}_V$ is a permutation of the vertices in $bg(\sigma_U)$. Because all vertices in $\sigma_O$ are sorted by a linear extension of $\prec_\mathcal{E}$, we also have that $\sigma_O \in \mathcal{O}_{bg(\sigma_U)}$ is a topological sorting of the vertices in $bg(\sigma_U)$. By Definition~\ref{def:bg}, we then have that $\sigma_O$ is an order-realization of $\sigma_U$: $\sigma_O \in \mathcal{R}_O(\sigma_U)$. Since, by construction, $l(t_i) \in \pi_a(v_i)$ if $\pi_o(v_i) = \:!$ and $l(t_i) \in \pi_a(v_i) \cup \{\tau\}$ if $\pi_o(v_i) = \:?$, we have that $\sigma = l(\sigma_T)\tproj_{\mathcal{U}_A} \in \mathcal{R}(\sigma_U)$. Since this construction is valid for any $\sigma \in \phi(bn(\sigma_U))$, every complete firing sequence of the behavior net is a realization of $\sigma_U$: $\phi(bn(\sigma_U)) \subseteq \mathcal{R}(\sigma_U)$.
	
	\leavevmode \\
	$(\supseteq)$ Let $\sigma_O \in \mathcal{R}_O(\sigma_U)$ be any order-realization of $\sigma_U$, and let $n = |\sigma_U|$. Since $\sigma_O[1] \prec_\mathcal{E} \sigma_O[2] \prec_\mathcal{E} \dots \prec_\mathcal{E} \sigma_O[n]$ (by Definition~\ref{def:real}), there exists a path $p \in P_{bg(\sigma_U)}$ such that $p = \langle v_1, v_2, \dots, v_n \rangle = \langle \sigma_O[1], \sigma_O[2], \dots, \sigma_O[n] \rangle$ (by Theorem~\ref{th:bg-correct}). Let $p_{1, 2} = (v_1, v_2)$, $p_{2, 3} = (v_2, v_3)$, and so on. Let $t_1 \in \mathbb{T}_{v_1}, t_2 \in \mathbb{T}_{v_2}, \dots, t_n \in \mathbb{T}_{v_n}$ and let $\sigma_T = \langle t_1, t_2, \dots, t_n \rangle$. By the construction in Definition~\ref{def:bn}, in $bn(\sigma_U) = N$ we have that
	$$(N, M_\mi{init}) [t_1\rangle (N, M_{1,2}) [t_2\rangle (N, M_{2, 3}) [t_3\rangle, \dots, [t_{n-1}\rangle (N, M_{n-1, n}) [t_n\rangle (N, M_\mi{final})$$
	where:
	\begin{gather*}
	M_{1, 2} = (M_\mi{start} \bmin [(\mi{start}, v_1)]) \bplus [p_{1, 2}]\\
	M_{2, 3} = (M_{1, 2} \bmin [p_{1, 2}]) \bplus [p_{2, 3}]\\
	\dots\\
	M_{n-1, n} = (M_{n-2, n-1} \bmin [p_{n-2, n-1}]) \bplus [p_{n-1, n}]\\
	M_\mi{final} = (M_{n-1, n} \bmin [p_{n-1, n}]) \bplus [(v_n, \mi{end})]\\
	\end{gather*}
	This construction implies that $(N, M_{\mi{init}})[\sigma_T \rhd (N,M_{\mi{final}})$ and therefore $\sigma_T \in \phi_f(bn(\sigma_U))$.
	
	The definition of the labeling function in the behavior net is such that, for all $1 \leq i \leq n$, we have that $(v_i, a) \in \mathbb{T}_{v_i} \Leftrightarrow a \in \pi_a(v_i)$. By Definition~\ref{def:real}, the labeling of the sequence $\langle t_1, t_2, \dots, t_n \rangle$ projected on the universe of activities is then a realization of the uncertain trace $\sigma_U$ obtained from the possible activity labels of $\sigma_O$: $l(\sigma_T)\tproj_{\mathcal{U}_A} = \mathcal{R}(\sigma_U)$. Since this construction is valid for any $\sigma_O \in \mathcal{R}_O(\sigma_U)$, the behavior net can replay any realization of $\sigma_U$: $\mathcal{R}(\sigma_U) \subseteq \phi(bn(\sigma_U))$. 
\end{proof}

\begin{theorem}[Correctness of uncertain alignments]
	Let $\sigma_U \in \mathcal{T}_U$ be a simple uncertain trace and let $\mi{SN} \in \mathcal{U}_{\mi{SN}}$ be a system net. Computing an alignment using the product net between $\mi{SN}$ and the behavior net $bn(\sigma_U)$ yields the alignment with the lowest cost among all realizations of $\sigma_U$: $\delta(\lambda_{\mi{SN}}(\sigma_{bn})) = \min_{\sigma \in \mathcal{R}(\sigma_U)} \lambda_{\mi{SN}}(\sigma) = \delta_{min}(\sigma_U)$.
\end{theorem}
\begin{proof}
	Recall from Definition~\ref{def:optalign} that $\lambda_{\mi{SN}} \colon \mathcal{T} \rightarrow {A_{\mi{LM}}}^* $ is a deterministic mapping that assigns any trace $\sigma$ to an optimal alignment. Adriansyah~\cite{adriansyah2014aligning} details how to compute such a function $\lambda_{\mi{SN}}$ through a state-based $\mathbb{A}^*$ search over a state space defined by the reachable markings of the product net $\mi{SN} \otimes \mi{en}(\sigma)$ between a reference system net $\mi{SN}$ and the event net a certain trace $\sigma \in \mathcal{T}$. As per Definition~\ref{def:alignment}, this search retrieves an alignment which is optimal with respect to a certain cost function $\delta$ and, ignoring ``$\nomove$'', is composed by a complete firing sequence of the system net $\sigma_T \in \phi_f(\mi{SN})$ and the only complete firing sequence of the event net $\mi{en}(\sigma)$, which corresponds to $\sigma$ by construction. Given a system net $\mi{SN} \in {\cal U}_{\mi{SN}}$, an uncertain trace $\sigma_U \in \mathcal{T}_U$ and its respective behavior net $bn(\sigma_U)$, the same search algorithm for $\lambda_{\mi{SN}}$ over $\mi{SN} \otimes \mi{bn}(\sigma_U)$ yields an optimal alignment containing a complete firing sequence for the reference system net $\sigma_T \in \phi_f(\mi{SN})$ and a complete firing sequence for the behavior net of the uncertain trace $\sigma \in \phi(\mi{bn}(\sigma_U))$. Since $\lambda_{\mi{SN}}$ minimizes the cost and $\sigma \in \mathcal{R}(\sigma_U)$ is a valid realization of $\sigma$ due to Theorem~\ref{th:bn-correct}, the resulting alignment has the minimal cost possible over all the possible realizations of the uncertain trace.
\end{proof}

\section{Experiments}\label{sec:experiments}
The framework for computing conformance bounds for uncertain event data illustrated in this paper raises some research questions that need to be addressed in a practical and empirical manner. The questions that we aim to answer are:
\begin{itemize}
	\item \emph{Q1}: how do conformance bounds behave when computed on uncertain data?
	\item \emph{Q2}: what is the impact of different deviating behavior and different types of uncertain behaviors on the conformance score of uncertain event logs?
	\item \emph{Q3}: what is the impact on the efficiency of computing uncertain alignments utilizing the behavior net as opposed to the baseline method of enumerating and aligning all realizations?
	\item \emph{Q4}: what is the impact on the efficiency of computing uncertain alignments utilizing the behavior net on different types of uncertain behavior?
	\item \emph{Q5}: how do trace length and the amount of uncertain events impact the intrinsic variability (i.e., the number of realizations) of uncertain event data?
	\item \emph{Q6}: is it possible to apply uncertain alignments to real-life data to obtain a best- and worst-case scenario for the execution of process instances?
\end{itemize}

The technique to compute conformance for strongly uncertain traces and to create the behavior net hereby described has been implemented in the Python programming language, thanks to the facilities for log importing, model creation and manipulation, and alignments provided by the library PM4Py~\cite{berti2019process}. Uncertainty has been represented in the XES standard through meta-attributes and constructs such as lists, such that any XES importer can read an uncertain log file. The algorithm was designed to be fully compatible with any event log in the XES format (both including and not including uncertainty); the meta-attributes for uncertainty were designed to be backward compatible with other process mining algorithms -- meta-attributes describing the possible values for an uncertain activity or the interval of an uncertain timestamp can also specify a ``fallback value'' which other process mining software will read as (certain) activity or timestamp value.

\subsection{Qualitative and Quantitative Experiments on Synthetic Data}

The first four research questions listed above have been addressed by tests on synthetic uncertain event logs. To this end, we implemented the following software components necessary to the experiments:
\begin{itemize}
	\item a \emph{noise generator}, to introduce deviations in a controlled way in an event log. This component allows to alter the activity label, swap the order of events or add redundant events to an event log with a given probability or frequency.
	\item an \emph{uncertainty generator}, to alter the XES attributes present in the log by appending additional meta-information which is then interpreted as uncertainty. The component introduces uncertainty information in an event log, with the possibility to add any of the strongly uncertain attributes described in the taxonomy of Section~\ref{sec:taxonomy}. This also allows for exporting the generated uncertain event log through the XES exporter of the PM4Py library.
	\item a number of smaller extensions to PM4Py functionalities, also useful for other process mining applications. Examples are the generation of all possible process variants (language) of a PM4Py Petri net, and a memoized version of alignments, which allows to trade off space in memory in order to speed up the computation of the conformance of an event log and a model.
\end{itemize}

The synthetic data generation and the software tools necessary to compute conformance bounds on uncertain event data are available within the PRocess mining OVer uncErtain Data (PROVED) project~\cite{pegoraro2021proved}. A specific branch of the repository hosting the project is dedicated to the experiments presented in this paper, making them readily reproducible\footnote{\url{https://github.com/proved-py/proved-core/tree/Conformance_Checking_over_Uncertain_Event_Data}}.

In order to answer \emph{Q1} and \emph{Q2}, we set up an experiment with the goal to inspect the quality of bounds for conformance scores as increasingly more uncertainty is added to an event log. We ran the tests on synthetic event logs where we added simulated uncertainty. In this way, we can control the amounts and types of uncertainty in event data.

Every iteration of this experiment is as follows:
\begin{enumerate}
	\item We generate a random Petri net with a fixed dimension ($n$ = 10 transitions) through the ProM plugin \emph{``Generate block-structured stochastic Petri nets''}.
	\item We play out an event log consisting of 100 traces generated from the Petri net.
	\item We randomly alter the activity label of a specific percentage $d_a$ of events, swapping it with another label sampled from the universe of activities.
	\item We randomly swap a specific percentage $d_s$ of events with their successor. For each event sampled for the swap, we randomly select either the predecessor or the successor (with 50\% probability each), and we swap the timestamps of the two events, effectively inverting their order. We skip the selection of the swap direction if we select the first event in a trace (which is swapped with the second) or the last event in a trace (which swaps with the second to last).
	\item We randomly duplicate a specific percentage $d_d$ of events. For each event selected for duplication, we create a new event in the trace with identical case ID and activity label, and with timestamp equal to the average between the timestamp of this selected event and the timestamp of the following event. If we select the last event in a trace for duplication, we simply add a fixed delta to the timestamp of the duplicate.
	\item We randomly introduce uncertainty in activity labels for a specific percentage $u_a$ of events. Each event selected for uncertainty on activity labels receives one additional activity label, different from the one it already has, sampled from the universe of activity labels.
	\item We randomly introduce uncertainty in timestamps for a specific percentage $u_t$ of events. For each event sampled for timestamp uncertainty we randomly choose either the predecessor or the successor (with 50\% probability each); the timestamp of the sampled event becomes an interval which extremes are the original timestamp and the timestamp of the predecessor or successor, effectively causing them to mutually overlap. In case the sampled event is the first (resp., last) event in a trace, we skip the selection of the predecessor or successor and we directly consider the successor (resp., predecessor) for the extremes of the uncertain timestamp.
	\item We randomly transform a specific percentage $u_i$ of events in indeterminate events. To these sampled events, we add the ``?'' attribute, in order to mark them as indeterminate.
	\item We measure upper and lower bounds for conformance score with increasing percentage $p$ of uncertainty.
\end{enumerate}

All sampling operations mentioned in the previous list are performed over a uniform probability distribution over the possible values.

In terms of amount of deviation to be considered in each configuration, we aimed at recreating a situation where there is significant deviating behavior with respect to the normative model; for each kind of deviation considered, we introduced anomalous behavior in 30\% of events. Thus, we consider four different settings for the addition of deviating behavior to events logs: \emph{Activity labels} = $\{d_a = 30\%, d_s = 0\%, d_d = 0\%\}$, \emph{Swaps} = $\{d_a = 0\%, d_s = 30\%, d_d = 0\%\}$, \emph{Extra events} = $\{d_a = 0\%, d_s = 0\%, d_d = 30\%\}$ and \emph{All} = $\{d_a = 30\%, d_s = 30\%, d_d = 30\%\}$.

We consider four different settings for the addition of uncertain behavior to events logs: \emph{Activities} = $\{u_a = p, u_t = 0\%, u_i = 0\%\}$, \emph{Timestamps} = $\{u_a = 0\%, u_t = p, u_i = 0\%\}$, \emph{Indeterminate events} = $\{u_a = 0\%, u_t = 0\%, u_i = p\}$ and \emph{All} = $\{u_a = p, u_t = p, u_i = p\}$. We test all four different configurations of deviation against each of the four configurations of uncertainty, with increasing values of $p$, for a total of 16 separate experiments.

Figure~\ref{fig:qual_exps} summarizes our findings. The plots on this figure represent the average of 10 runs as described above.

We can observe that, in general, all plots show the expected behavior: the upper and lower bound for conformance coincide at percentage of uncertain events $p$ = 0 for all experiments, to then diverge while $p$ increases.
A number of additional observations can be made looking at individual configurations for deviation or uncertainty, or at specific scatter plots. When only uncertainty on activity labels is added to the event log, we see a deterioration of the upper bound for conformance cost, but the lower bound does not improve -- in fact, it is essentially constant. This can be attributed to the fact that, since to generate uncertainty on activity label we sample from the set of labels randomly, the chances of observing a realization of a trace where an uncertain activity label matches the alteration introduced by the deviations are small. Uncertainty on timestamps makes the lower bound decrease only when the introduced deviations are swaps: as expected, the possibility of changing the order of pairs of events does not have a sensible improvement in the lower bound for deviation when extra events are added or activity labels of existing events are altered.

Conversely, the possibility to ``skip'' some critical events has a positive effect on the lower bound of all possible configurations for deviations: in fact, when marking some events as indeterminate in a log where extra events were added as deviations, the average conformance cost drops by 30.61\% at $p$ = 16\%, the largest drop among all the experiments. The experiment with all three types of uncertainty and extra events as deviations essentially displays the same effect (improvement in lower bound is slightly lower, but not significantly, with a decrease in deviation of 29.38\% at $p$ = 16\%).

For the experiments where all types of deviations were added at once, we can see that, as could be anticipated, the differences in deviation scores on the two bounds become smaller in relative terms (because of the very high amount of deviations at $p$ = 0\%), but larger in absolute terms. As per the previous experiments, the largest contributor in decreasing the conformance cost of the lower bound is the addition of indeterminate events, which by itself decreases the deviation cost by 13.92\% at $p$ = 16\%. In general, the vast variability in measuring the conformance of an uncertain log shows that, if all types of uncertainty can occur with high frequency in a process, the business owner should act on the uncertainty sources, since they will be a major obstacle in obtaining accurate measurements of process conformance. Vice versa, in the case of limited occurrences of uncertainty in event data, the algorithm here proposed is able to provide actionable bounds for conformance score, together with descriptions of best- and worst-case scenarios of process conformance for a given trace.

\begin{sidewaysfigure}[h]
	\includegraphics[width=\textwidth]{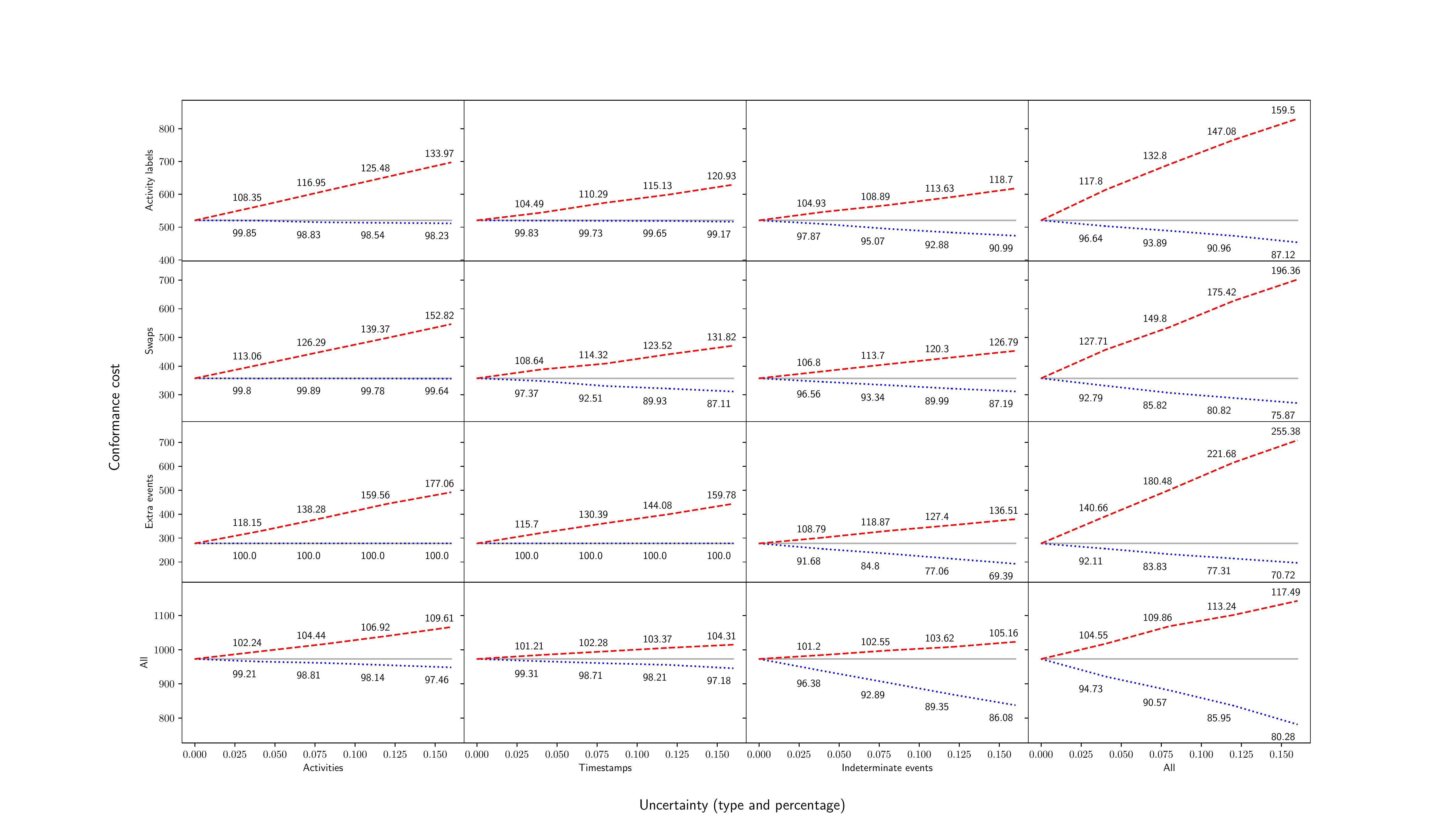}
	\caption{Upper (red, dashed) and lower (blue, dotted) bound for conformance cost for synthetic event logs with increasing uncertainty. Every plot shows a different configuration of deviation added to the log and types of uncertainty simulated in the event data. The x-axis shows the percentage of uncertainty $p$ added to the logs; the y-axis shows the amount of deviations, computed with alignments. The labels inside the graph indicate the relative change in deviation score with respect to $p$ = 0, as a percentage. The gray continuous lines indicate the amounts of deviations at $p$ = 0 as a reference.}
	\label{fig:qual_exps}
\end{sidewaysfigure}

\clearpage

The second experiment we set up aims to answer questions \emph{Q3} and \emph{Q4}, and is concerned with the performance of calculating the lower bound of the cost via the behavior net versus the brute-force method of listing all the realizations of an uncertain trace, evaluating all of them through alignments, then picking the best value. We used a constant percentage of uncertain events of $p$ = 5\% and logs of 100 traces for this test, with progressively increasing values of $n$. We ran 4 different experiments, each with one of the four configurations for uncertain behavior \emph{Activities}, \emph{Timestamps}, \emph{Indeterminate events} and \emph{All} illustrated above.

\begin{figure}[h]
	\centering
	\includegraphics[width=\textwidth]{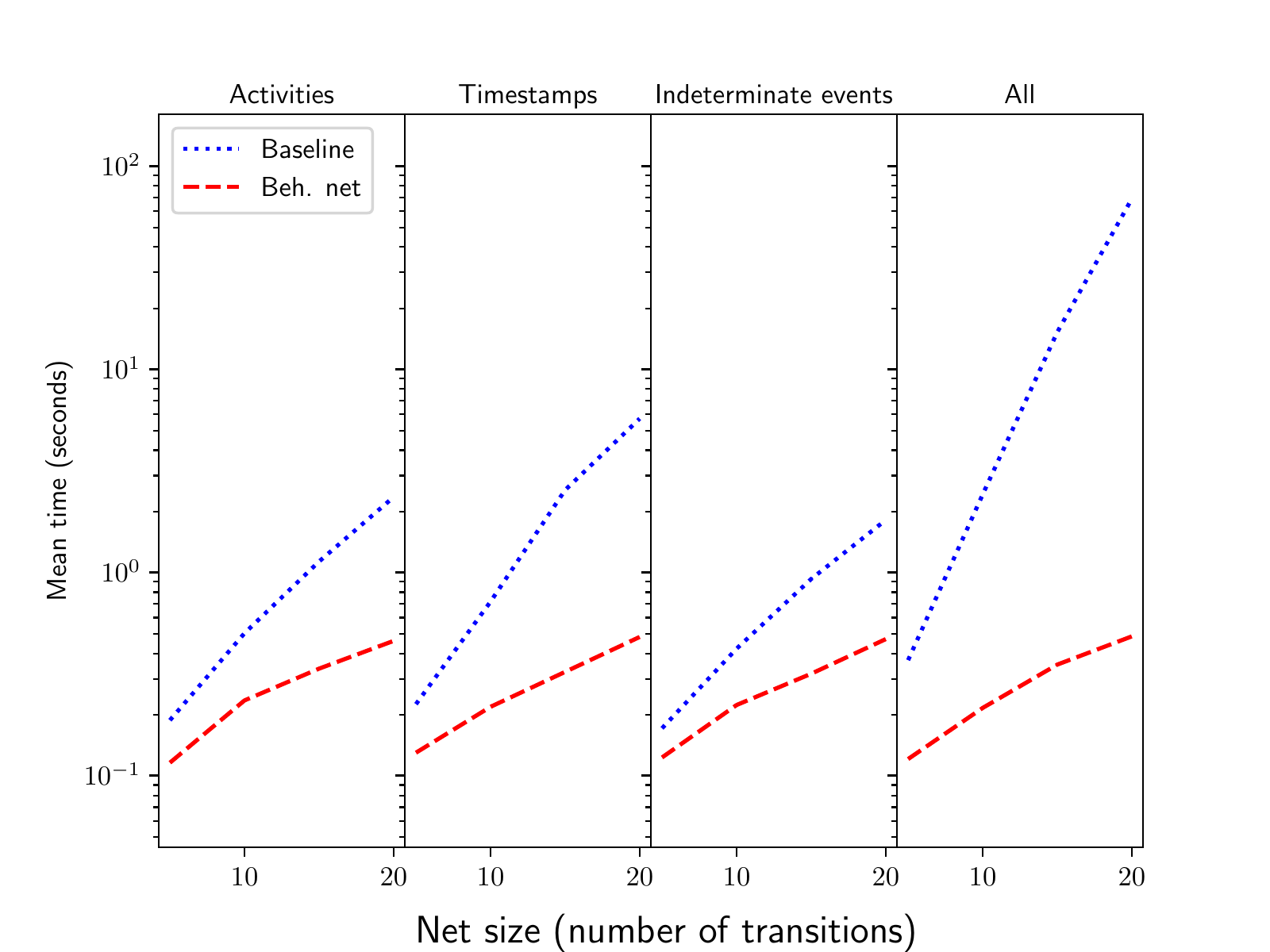}
	\caption{Effect on the time performance of calculating the lower bound for conformance cost with the brute-force method (blue) vs. the behavior net (red) on four different configurations for uncertain events.}
	\label{fig:performance}
\end{figure}

Figure~\ref{fig:performance} summarizes the results. As the diagram shows, the difference in time between the two methods tends to diverge quickly even on a logarithmic scale. The largest model we could test was $n$ = 20, a Petri net with 20 transitions, which is comparatively tiny in practical terms; however, even at these small scales the brute-force method takes roughly 3 orders of magnitude more than the time needed by the behavior net, when all the types of uncertainty are added with $p$ = 5\%.

This shows a very large improvement in the computing time for the lower bound computation; thus, the best-case scenario for the conformance cost of an uncertain trace can be obtained efficiently thanks to the structural properties of the behavior net. This graph also shows the dramatic impact on the number of realizations of a behavior net -- and thus, the time needed to perform a brute-force computation of alignments -- when the effects of different kinds of uncertainty are compounded.

Let us now answer \emph{Q5}. In order to assess the impact of uncertainty on the variability of event data and, consequently, on the performance of uncertain process mining techniques, we computed the total number of realizations in an uncertain event log of 100 traces. Figure~\ref{fig:num_real_size} shows the change in number of realizations with the increase in size of the Petri net used to generate the log, while the percentage of uncertain events is constant (5\%). Conversely, Figure~\ref{fig:num_real_unc} shows the change in the number of realizations with the increase in percentage of uncertain events, while the size of the Petri nets is fixed (10 transitions).

\begin{figure}[t]
	\centering
	\includegraphics[width=.7\textwidth, keepaspectratio]{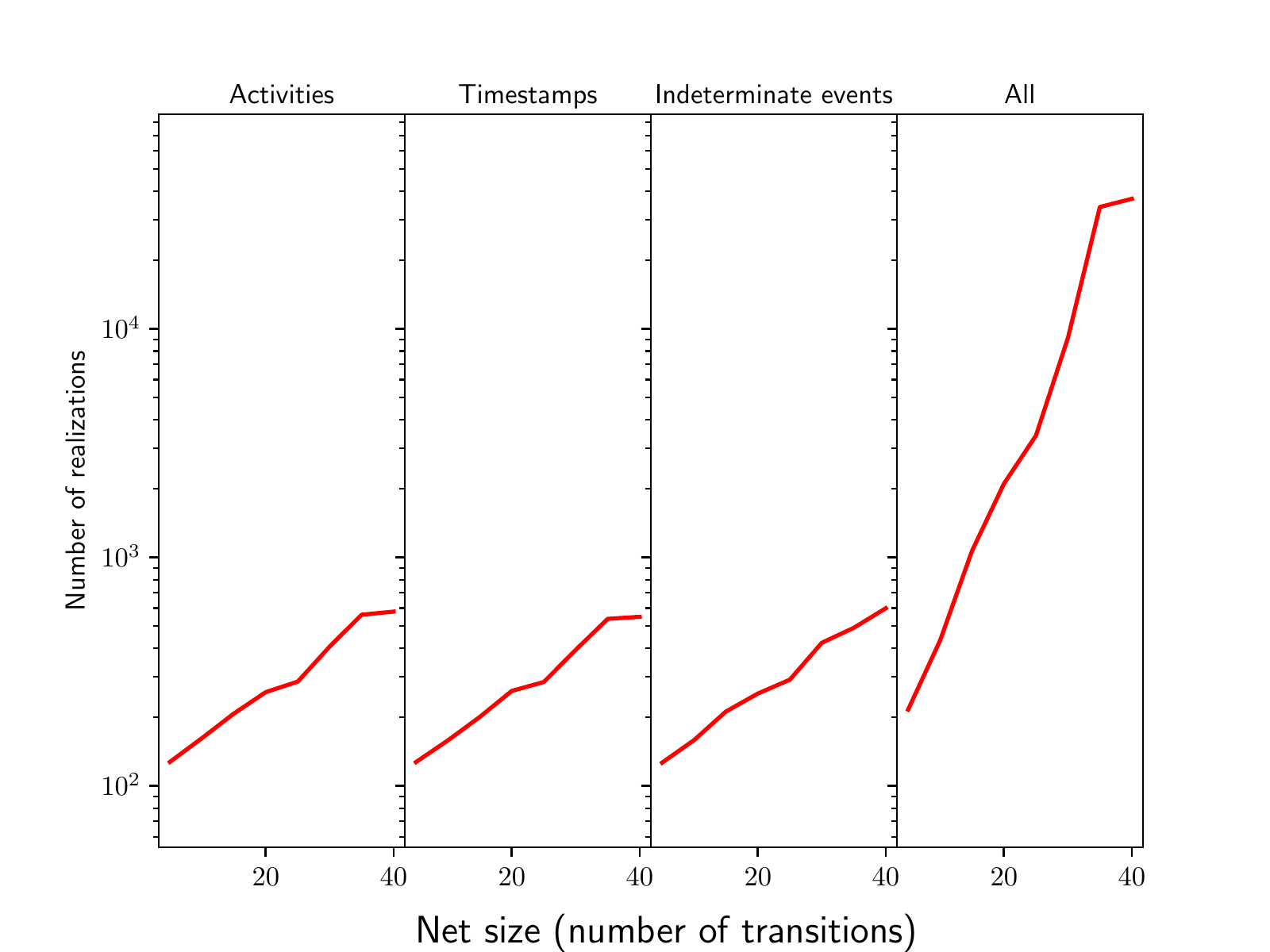}
	\caption{Number of realizations (average per log over 10 randomly generated logs of 100 traces) in an uncertain event log in function of the size of the Petri net used to generate it. The logs are generated through simulation with 10 different and randomly wired Petri nets of increasing size. Four different configurations for uncertainty are shown: on activities, timestamps, indeterminate events, and all three combined. Uncertainty is introduced in the events within the log in a fixed proportion of 5\%.}
	\label{fig:num_real_size}
\end{figure}

Figures~\ref{fig:num_real_size} and~\ref{fig:num_real_unc} justify the results of the experiments on performance shown in Figure~\ref{fig:performance}: there is a clear exponential relation between the number of realizations resulting from an uncertain event log and both trace length and percentage of uncertain events in the log. Both factors, when increasing, induce an exponential increase in the total number of realizations even when considered separately. Specifically, the comparison of Figures~\ref{fig:performance} and~\ref{fig:num_real_size} highlights the cause of the inefficiency of computing alignments for every realization of an uncertain trace, evidently showing the linear relationship between the number of realizations in an uncertain log and the time expenditure of the brute-force alignments approach.

\begin{figure}[H]
	\centering
	\includegraphics[width=.7\textwidth, keepaspectratio]{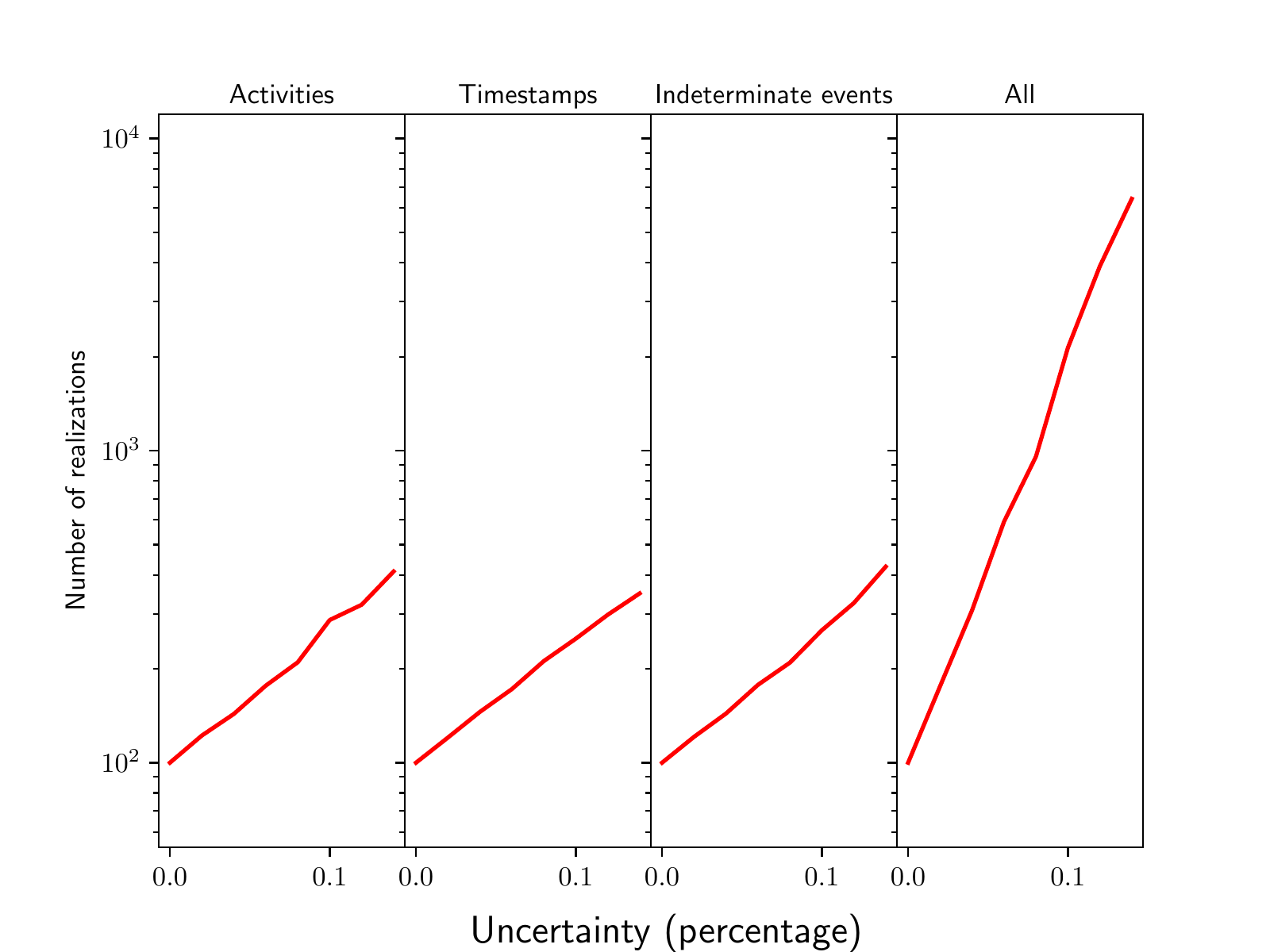}
	\caption{Number of realizations (average per log over 10 randomly generated logs of 100 traces) in an uncertain event log in function of the percentage of uncertain events within the log. The logs are generated through simulation with 10 different and randomly wired Petri nets with a fixed size of 10 transitions. Four different configurations for uncertainty are shown: on activities, timestamps, indeterminate events, and all three combined.}
	\label{fig:num_real_unc}
\end{figure}

\subsection{Applications on Real-Life Data}
As illustrated in Section~\ref{sec:introduction}, uncertainty in event data can originate from a number of different causes in real-world applications. One prominent source of uncertainty is missing data: attribute values not recorded in an event log can on occasions be described by uncertainty, through domain knowledge provided by process owners or experts. Then, as described in this paper, it is possible to obtain a detailed analysis of the deviations of a best- and worst-case scenario for the conformance to a process model.

To seek to answer research question \emph{Q6} through a direct application of conformance checking over uncertainty, let us consider a process related to the medical procedures performed in the Intensive Care Unit (ICU) of a hospital. Figure~\ref{fig:procmodel} shows a ground truth model for the process.

\begin{figure}[t]
	\centering
	\begin{adjustbox}{center}
		\begin{tikzpicture}[node distance=.7cm and .35cm, >=stealth']
		
		\tikzstyle{place} = [circle,draw,thick,minimum size=6mm]
		\tikzstyle{transition} = [rectangle,draw,thick,minimum size=4mm]
		\tikzstyle{invisible} = [transition, fill=black]
		\tikzstyle{finaltoken} = [token, fill=black!30]
		
		\node [place,tokens=1] (p1) [] {};
		
		\node [transition] (t1) [below= of p1,label=left:{\scriptsize $t_1$}] {Access};
		\draw [->] (p1) to (t1.north);
		
		\node [place] (p2) [below= of t1] {};
		\draw [->] (t1.south) to (p2);
		
		\node [transition] (t2) [below= of p2,label=below:{\scriptsize $t_2$}] {Triage};
		\draw [->] (p2) to (t2.north);
		
		\node [place] (p3) [above right= of t2] {};
		\node [place] (p4) [below right= of t2] {};
		\node [place] (p5) [above= 1.4cm of p3] {};
		\node [place] (p6) [below= 1.4cm of p4] {};
		\draw [->] (t2.east) to (p3);
		\draw [->] (t2.east) to (p4);
		\draw [->] (t2.east) to (p5);
		\draw [->] (t2.east) to (p6);
		
		\node [transition] (t8) [above right= .1cm and .7cm of p6,label=above:{\scriptsize $t_3$}] {R1};
		\draw [->] (p6) to (t8.west);
		
		\node [place] (p7) [right= of t8] {};
		\draw [->] (t8.east) to (p7);
		
		\node [transition] (t9) [right= of p7,label=above:{\scriptsize $t_4$}] {R2};
		\draw [->] (p7) to (t9.west);
		
		\node [place] (p8) [right= of t9] {};
		\draw [->] (t9.east) to (p8);
		
		\node [transition] (t10) [right= of p8,label=above:{\scriptsize $t_5$}] {R3};
		\draw [->] (p8) to (t10.west);
		
		\node [place] (p9) [right= of t10] {};
		\draw [->] (t10.east) to (p9);
		
		\node [transition] (t11) [right= of p9,label=above:{\scriptsize $t_6$}] {R4};
		\draw [->] (p9) to (t11.west);
		
		\node [place] (p10) [below right= .1cm and .7cm of t11] {};
		\draw [->] (t11.east) to (p10);
		
		\node [invisible] (t12) at ($(p6)!0.5!(p10)$) [below= .1cm,label=above:{\scriptsize $t_7$}] {Laboratory};
		\draw [->] (p6) to (t12);
		\draw [->] (t12) to (p10);
		
		\node [place] (p11) [above= 1.4cm of p10] {};
		
		\node [place] (p12) [above= 1.7cm of p11] {};
		
		\node [place] (p13) [above= 1.4cm of p12] {};
		
		\node [transition] (t3) at ($(p5)!0.5!(p13)$) [,label=above:{\scriptsize $t_8$}] {Visit};
		\draw [->] (p5) to (t3.west);
		\draw [->] (t3.east) to (p13);
		
		\node [place] (p31) at ($(p3)!0.5!(p12)$) [above= .4cm] {};
		
		\node [transition] (t4) at ($(p3)!0.5!(p31)$) [above= .1cm,label=above:{\scriptsize $t_9$}] {ConsultancyBegin};
		\draw [->] (p3) to (t4.west);
		\draw [->] (t4.east) to (p31);
		
		\node [transition] (t41) at ($(p31)!0.5!(p12)$) [above= .1cm,label=above:{\scriptsize $t_{10}$}] {ConsultancyEnd};
		\draw [->] (p31) to (t41.west);
		\draw [->] (t41.east) to (p12);
		
		\node [invisible] (t5) at ($(p3)!0.5!(p12)$) [below= .1cm,label=above:{\scriptsize $t_{11}$}] {Consultancy};
		\draw [->] (p3) to (t5.west);
		\draw [->] (t5.east) to (p12);
		
		\node [place] (p41) at ($(p4)!0.5!(p11)$) [above= .4cm] {};
		
		\node [transition] (t6) at ($(p4)!0.5!(p41)$) [above= .1cm,label=above:{\scriptsize $t_{12}$}] {Laboratory - Begin};
		\draw [->] (p4) to (t6.west);
		\draw [->] (t6.east) to (p41);
		
		\node [transition] (t61) at ($(p41)!0.5!(p11)$) [above= .1cm,label=above:{\scriptsize $t_{13}$}] {Laboratory - End};
		\draw [->] (p41) to (t61.west);
		\draw [->] (t61.east) to (p11);
		
		\node [invisible] (t7) at ($(p4)!0.5!(p11)$) [below= .1cm,label=above:{\scriptsize $t_{14}$}] {Laboratory};
		\draw [->] (p4) to (t7.west);
		\draw [->] (t7.east) to (p11);
		
		\node [transition] (t8) [above right= of p11,label=below:{\scriptsize $t_{15}$}] {Dismissal};
		\draw [->] (p10) to (t8.west);
		\draw [->] (p11) to (t8.west);
		\draw [->] (p12) to (t8.west);
		\draw [->] (p13) to (t8.west);
		
		\node [place] (p14) [above= of t8] {};
		\draw [->] (t8.north) to (p14);
		
		\node [transition] (t9) [above= of p14,label=right:{\scriptsize $t_{16}$}] {Exit};
		\draw [->] (p14) to (t9.south);
		
		\node [place] (p15) [above= of t9] {};
		\draw [->] (t9.north) to (p15);
		\node[finaltoken] at (p15) {};
		
		\end{tikzpicture}
	\end{adjustbox}
	\caption{The Petri net that models the process related to the treatment of patients in the ICU ward of an Italian hospital. The activities R1 through R4 are abbreviations for the four phases of a radiology exam: respectively, \emph{Radiology - Submitted Request}, \emph{Radiology - Accepted Request}, \emph{Radiology - Exam}, \emph{Radiology - Results}.}
	\label{fig:procmodel}
\end{figure}
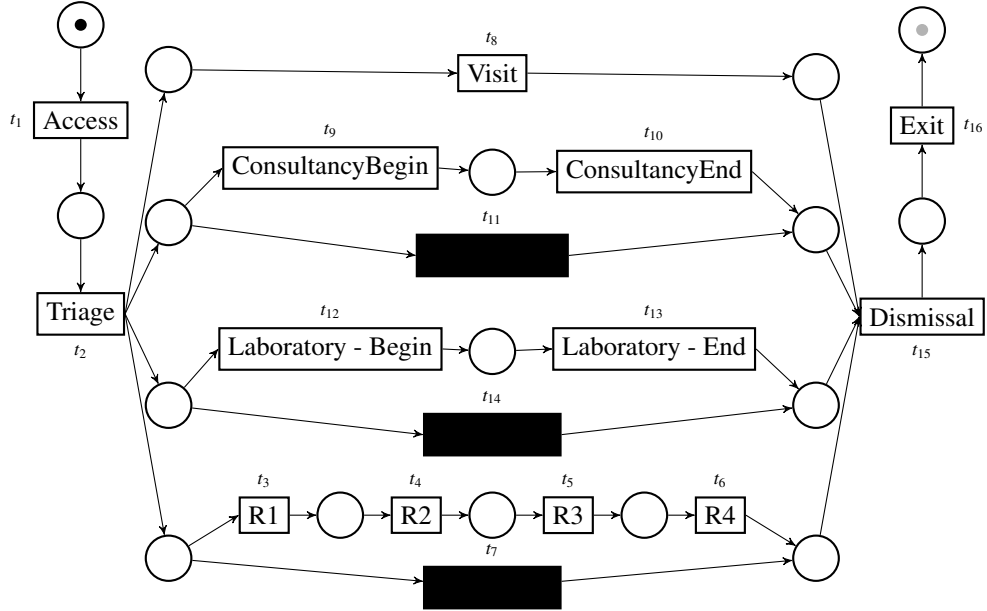

An execution log containing events that concern this ICU process is available. Throughout the process, some anomalies with attribute values can be spotted -- namely, a number of anomalies affecting the timestamp attributes. This is a $[\text{E}]_\mathbb{S}$-type uncertain log.

The alterations on the timestamps in this event log happen for a number of reasons. The domain experts reported that human error is a frequent source of anomaly, which is worsened by the fact that operators often do not input data in real-time, but the information is recorded after a certain delay (e.g., at the end of a shift). Moreover, the information systems of the ICU ward and other wards (such as radiology, for instance) do not allow for automatic transmission of data between one another, so in some occurrences the timestamp of visits by specialists is not recorded in the ICU information system.

Tables~\ref{table:uncertaintraceicu1} and~\ref{table:uncertaintraceicu2} show two examples of traces with anomalous timestamp behavior. We can see that in the trace of Table~\ref{table:uncertaintraceicu1} the event \emph{Triage} has an imprecise timestamp -- only the day has been recorded. This can be modeled with an uncertain timestamp encompassing a range of 24 hours. The column \emph{Preprocessed Timestamp} shows the results of this preprocessing step.

\begin{table}[H]
	\caption{Events related to one case of the ICU process. The timestamp of the ``Triage'' event is imprecise: through domain knowledge, we are able to represent this uncertainty in an explicit way within the event attributes in the log.}
	\label{table:uncertaintraceicu1}
	\centering
	\small
	\begin{adjustbox}{center}
		\begin{tabular}{cccc}
			\textbf{Event ID}				& \textbf{Raw Timestamp}					& \textbf{Preprocessed Timestamp}										& \textbf{Activity}									\\ \hline
			\multicolumn{1}{|c|}{$e_1$}		& \multicolumn{1}{c|}{2017-02-20 23:59:31}	& \multicolumn{1}{c|}{2017-02-20 23:59:31}							& \multicolumn{1}{c|}{\emph{Access}}				\\ \hline
			\multicolumn{1}{|c|}{$e_2$}		& \multicolumn{1}{c|}{2017-02-21 00:02:58}	& \multicolumn{1}{c|}{2017-02-21 00:02:58}							& \multicolumn{1}{c|}{\emph{Visit}}					\\ \hline
			\multicolumn{1}{|c|}{$e_3$}		& \multicolumn{1}{c|}{2017-02-21 00:06:30}	& \multicolumn{1}{c|}{2017-02-21 00:06:30}							& \multicolumn{1}{c|}{\emph{ConsultancyBegin}}	\\ \hline
			\multicolumn{1}{|c|}{$e_4$}		& \multicolumn{1}{c|}{2017-02-21 00:29:12}	& \multicolumn{1}{c|}{2017-02-21 00:29:12}							& \multicolumn{1}{c|}{\emph{R1}}					\\ \hline
			\multicolumn{1}{|c|}{$e_5$}		& \multicolumn{1}{c|}{2017-02-21 00:41:00}	& \multicolumn{1}{c|}{2017-02-21 00:41:00}							& \multicolumn{1}{c|}{\emph{R2}}					\\ \hline
			\multicolumn{1}{|c|}{$e_6$}		& \multicolumn{1}{c|}{2017-02-21 00:41:00}	& \multicolumn{1}{c|}{2017-02-21 00:41:00}							& \multicolumn{1}{c|}{\emph{R3}}					\\ \hline
			\multicolumn{1}{|c|}{$e_7$}		& \multicolumn{1}{c|}{2017-02-21 01:02:00}	& \multicolumn{1}{c|}{2017-02-21 01:02:00}							& \multicolumn{1}{c|}{\emph{R4}}					\\ \hline
			\multicolumn{1}{|c|}{$e_8$}		& \multicolumn{1}{c|}{2017-02-21 01:56:26}	& \multicolumn{1}{c|}{2017-02-21 01:56:26}							& \multicolumn{1}{c|}{\emph{ConsultancyEnd}}		\\ \hline
			\multicolumn{1}{|c|}{$e_9$}		& \multicolumn{1}{c|}{2017-02-21 02:01:37}	& \multicolumn{1}{c|}{2017-02-21 02:01:37}							& \multicolumn{1}{c|}{\emph{Dismissal}}				\\ \hline
			\multicolumn{1}{|c|}{$e_{10}$}	& \multicolumn{1}{c|}{2017-02-21 02:02:36}	& \multicolumn{1}{c|}{2017-02-21 02:02:36}							& \multicolumn{1}{c|}{\emph{Exit}}					\\ \hline
			\multicolumn{1}{|c|}{$e_{11}$}	& \multicolumn{1}{c|}{2017-02-21}			& \multicolumn{1}{c|}{[2017-02-21 00:00:00, 2017-02-21 23:59:59]}	& \multicolumn{1}{c|}{\emph{Triage}}				\\ \hline
		\end{tabular}
	\end{adjustbox}
	\normalsize
\end{table}

Some of the events in the trace of Table~\ref{table:uncertaintraceicu2} are missing the timestamp value entirely. In this case, we can resort to domain knowledge provided by the process owners: it is known that events related to the \emph{Radiology} exams happen after the \emph{Triage} event, and before the \emph{Dismissal} event. This allows the representation of the timestamps with ranges of possible values. Notice that such a small interval of time, obtainable from the domain knowledge available, is preferable to larger possible intervals (e.g., 2017-08-27 00:00:00 to 2017-08-27 23:59:59), since it minimizes the amount of possible overlaps in time with other events in the trace. In turn, this means that the number of possible realizations of the uncertain trace is smaller, granting a faster conformance checking. As before, the results of modeling timestamp uncertainty are shown in the column \emph{Preprocessed Timestamp}.

\begin{table}[]
	\caption{Events related to one case of the ICU process. Some of the timestamp attributes are missing: through domain knowledge, we are able to represent them with uncertainty within a small interval of time. The timestamps in bold and italic of the ``Raw Timestamp'' column are used to set the interval boundaries for uncertain timestamps.}
	\label{table:uncertaintraceicu2}
	\centering
	\small
	\begin{adjustbox}{center}
		\begin{tabular}{cccc}
			\textbf{Event ID}				& \textbf{Raw Timestamp}					& \textbf{Preprocessed Timestamp}										& \textbf{Activity}									\\ \hline
			\multicolumn{1}{|c|}{$e_1$}		& \multicolumn{1}{c|}{2017-08-27 11:47:46}	& \multicolumn{1}{c|}{2017-08-27 11:47:46}							& \multicolumn{1}{c|}{\emph{Access}}				\\ \hline
			\multicolumn{1}{|c|}{$e_2$}		& \multicolumn{1}{c|}{2017-08-27 \textbf{11:47:53}}	& \multicolumn{1}{c|}{2017-08-27 11:47:53}							& \multicolumn{1}{c|}{\emph{Triage}}				\\ \hline
			\multicolumn{1}{|c|}{$e_3$}		& \multicolumn{1}{c|}{2017-08-27 12:14:25}	& \multicolumn{1}{c|}{2017-08-27 12:14:25}							& \multicolumn{1}{c|}{\emph{Visit}}					\\ \hline
			\multicolumn{1}{|c|}{$e_4$}		& \multicolumn{1}{c|}{2017-08-27 12:33:24}	& \multicolumn{1}{c|}{2017-08-27 12:33:24}							& \multicolumn{1}{c|}{\emph{R1}}					\\ \hline
			\multicolumn{1}{|c|}{$e_5$}		& \multicolumn{1}{c|}{2017-08-27 13:04:11}	& \multicolumn{1}{c|}{2017-08-27 13:04:11}							& \multicolumn{1}{c|}{\emph{ConsultancyBegin}}	\\ \hline
			\multicolumn{1}{|c|}{$e_6$}		& \multicolumn{1}{c|}{2017-08-27 \emph{13:04:53}}	& \multicolumn{1}{c|}{2017-08-27 13:04:53}							& \multicolumn{1}{c|}{\emph{Dismissal}}				\\ \hline
			\multicolumn{1}{|c|}{$e_7$}		& \multicolumn{1}{c|}{2017-08-27 13:08:07}	& \multicolumn{1}{c|}{2017-08-27 13:08:07}							& \multicolumn{1}{c|}{\emph{Exit}}					\\ \hline
			\multicolumn{1}{|c|}{$e_8$}		& \multicolumn{1}{c|}{\textbf{NULL}}		& \multicolumn{1}{c|}{[2017-08-27 \textbf{11:47:53}, 2017-08-27 \emph{13:04:53}]}	& \multicolumn{1}{c|}{\emph{ConsultancyEnd}}		\\ \hline
			\multicolumn{1}{|c|}{$e_9$}		& \multicolumn{1}{c|}{\textbf{NULL}}		& \multicolumn{1}{c|}{[2017-08-27 \textbf{11:47:53}, 2017-08-27 \emph{13:04:53}]}	& \multicolumn{1}{c|}{\emph{R2}}					\\ \hline
			\multicolumn{1}{|c|}{$e_{10}$}	& \multicolumn{1}{c|}{\textbf{NULL}}		& \multicolumn{1}{c|}{[2017-08-27 \textbf{11:47:53}, 2017-08-27 \emph{13:04:53}]}	& \multicolumn{1}{c|}{\emph{R3}}					\\ \hline
			\multicolumn{1}{|c|}{$e_{11}$}	& \multicolumn{1}{c|}{\textbf{NULL}}		& \multicolumn{1}{c|}{[2017-08-27 \textbf{11:47:53}, 2017-08-27 \emph{13:04:53}]}	& \multicolumn{1}{c|}{\emph{R4}}					\\ \hline
		\end{tabular}
	\end{adjustbox}
	\normalsize
\end{table}

Once uncertainty is made explicit using the event log formally defined in this paper, it is possible to apply conformance checking over uncertainty. The technique of alignments illustrated here provides two results, corresponding to the lower and upper bound for the conformance score. The traces shown in Tables~\ref{table:uncertaintraceicu1} and~\ref{table:uncertaintraceicu2} have a best-case scenario alignment in common, which is shown in Table~\ref{table:bestalignmentsuncertaintracesicu}; aligning through the behavior net of these traces has allowed the algorithm to select a value for the uncertain timestamps of the traces (translated in a specific ordering) such that the deviations between data and model is the smallest possible. For both traces, the best-case scenario has a cost equal to 0, thus, no deviations occur in that case.

\begin{table}[]
	\caption{A valid alignment for both traces of Tables~\ref{table:uncertaintraceicu1} and~\ref{table:uncertaintraceicu2}. This alignment has a deviation cost equal to 0, and corresponds to a best-case scenario for conformance between the process model and both uncertain traces.}
	\label{table:bestalignmentsuncertaintracesicu}
	\centering
	\scriptsize
	\begin{tabular}{cccccccccccc}
		\multicolumn{1}{|c|}{Access}	& \multicolumn{1}{c|}{Triage}	& \multicolumn{1}{c|}{Visit}	& \multicolumn{1}{c|}{ConsultancyBegin}	& \multicolumn{1}{c|}{R1}		& \multicolumn{1}{c|}{R2}		& \multicolumn{1}{c|}{R3}		& \multicolumn{1}{c|}{R4}		& \multicolumn{1}{c|}{ConsultancyEnd}	& \multicolumn{1}{c|}{$\nomove$}		& \multicolumn{1}{c|}{Dismissal}	& \multicolumn{1}{c|}{Exit}		\\ \hline
		\multicolumn{1}{|c|}{Access}	& \multicolumn{1}{c|}{Triage}	& \multicolumn{1}{c|}{Visit}	& \multicolumn{1}{c|}{ConsultancyBegin}	& \multicolumn{1}{c|}{R1}		& \multicolumn{1}{c|}{R2}		& \multicolumn{1}{c|}{R3}		& \multicolumn{1}{c|}{R4}		& \multicolumn{1}{c|}{ConsultancyEnd}	& \multicolumn{1}{c|}{$\tau$}			& \multicolumn{1}{c|}{Dismissal}	& \multicolumn{1}{c|}{Exit}		\\ 
		\multicolumn{1}{|c|}{$t_1$}		& \multicolumn{1}{c|}{$t_2$}	& \multicolumn{1}{c|}{$t_8$}	& \multicolumn{1}{c|}{$t_9$}				& \multicolumn{1}{c|}{$t_3$}	& \multicolumn{1}{c|}{$t_4$}	& \multicolumn{1}{c|}{$t_5$}	& \multicolumn{1}{c|}{$t_6$}	& \multicolumn{1}{c|}{$t_{10}$}				& \multicolumn{1}{c|}{$t_{14}$}			& \multicolumn{1}{c|}{$t_{15}$}		& \multicolumn{1}{c|}{$t_{16}$}	\\ 
	\end{tabular}
	\normalsize
\end{table}

Let us now look at the worst-case scenarios. One of the alignments with the worst possible score for the trace in Table~\ref{table:uncertaintraceicu1} is shown in Table~\ref{table:worstalignmentsuncertaintraceicu1}. In this scenario, the deviations are one move on model (the \emph{Triage} activity should have occurred after the \emph{Access} but did not), and one move on log (the activity \emph{Triage} occurs in the data at an unexpected moment in the process).

\begin{table}[]
	\caption{A valid alignment for the trace of Table~\ref{table:uncertaintraceicu1}. This alignment has a deviation cost equal to 2 (1 move on log and 1 move on model), and corresponds to a worst-case scenario for conformance between the process model and the uncertain trace.}
	\label{table:worstalignmentsuncertaintraceicu1}
	\centering
	\scriptsize
	\begin{adjustbox}{center}
		\begin{tabular}{ccccccccccccc}
			\multicolumn{1}{|c|}{Access}	& \multicolumn{1}{c|}{$\nomove$}	& \multicolumn{1}{c|}{Visit}	& \multicolumn{1}{c|}{ConsultancyBegin}	& \multicolumn{1}{c|}{R1}		& \multicolumn{1}{c|}{R2}		& \multicolumn{1}{c|}{R3}		& \multicolumn{1}{c|}{R4}		& \multicolumn{1}{c|}{ConsultancyEnd}	& \multicolumn{1}{c|}{$\nomove$}		& \multicolumn{1}{c|}{Dismissal}	& \multicolumn{1}{c|}{Exit}		& \multicolumn{1}{c|}{Triage}		\\ \hline
			\multicolumn{1}{|c|}{Access}	& \multicolumn{1}{c|}{Triage}		& \multicolumn{1}{c|}{Visit}	& \multicolumn{1}{c|}{ConsultancyBegin}	& \multicolumn{1}{c|}{R1}		& \multicolumn{1}{c|}{R2}		& \multicolumn{1}{c|}{R3}		& \multicolumn{1}{c|}{R4}		& \multicolumn{1}{c|}{ConsultancyEnd}	& \multicolumn{1}{c|}{$\tau$}			& \multicolumn{1}{c|}{Dismissal}	& \multicolumn{1}{c|}{Exit}		& \multicolumn{1}{c|}{$\nomove$}	\\ 
			\multicolumn{1}{|c|}{$t_1$}		& \multicolumn{1}{c|}{$t_2$}		& \multicolumn{1}{c|}{$t_8$}	& \multicolumn{1}{c|}{$t_9$}				& \multicolumn{1}{c|}{$t_3$}	& \multicolumn{1}{c|}{$t_4$}	& \multicolumn{1}{c|}{$t_5$}	& \multicolumn{1}{c|}{$t_6$}	& \multicolumn{1}{c|}{$t_{10}$}				& \multicolumn{1}{c|}{$t_{14}$}			& \multicolumn{1}{c|}{$t_{15}$}		& \multicolumn{1}{c|}{$t_{16}$}	& \multicolumn{1}{c|}{}		\\ 
		\end{tabular}
	\end{adjustbox}
	\normalsize
\end{table}

A worst-case scenario for the trace in Table~\ref{table:uncertaintraceicu2} is illustrated in Table~\ref{table:worstalignmentsuncertaintraceicu2}. In this case, the deviation is equal to 6, given by the wrong order of the event related to the \emph{Radiology} exam. Note that, in this example, we assume that every deviation has a unit cost, but the alignment technique allows to define different costs for different types of deviations based on impact in the process. For example, a patient that exits the hospital without official dismissal might have a worse impact than an unauthorized laboratory exam. For simplicity, in this case, we assume that all types of deviation have a unit cost.

\begin{table}[h]
	\caption{A valid alignment for the trace of Table~\ref{table:uncertaintraceicu2}. This alignment has a cost equal to 6 (3 moves on log and 3 moves on model), and corresponds to the worst-case scenario for conformance between the process model and the uncertain trace.}
	\label{table:worstalignmentsuncertaintraceicu2}
	\centering
	\scriptsize
	\begin{adjustbox}{center}
		\begin{tabular}{ccccccccccccccc}
			\multicolumn{1}{|c|}{Access}	& \multicolumn{1}{c|}{Triage}	& \multicolumn{1}{c|}{Visit}	& \multicolumn{1}{c|}{ConsultancyBegin}	& \multicolumn{1}{c|}{$\nomove$}	& \multicolumn{1}{c|}{$\nomove$}	& \multicolumn{1}{c|}{$\nomove$}	& \multicolumn{1}{c|}{R4}		& \multicolumn{1}{c|}{R3}		& \multicolumn{1}{c|}{R2}		& \multicolumn{1}{c|}{R1}		& \multicolumn{1}{c|}{ConsultancyEnd}	& \multicolumn{1}{c|}{$\nomove$}		& \multicolumn{1}{c|}{Dismissal}	& \multicolumn{1}{c|}{Exit}		\\ \hline
			\multicolumn{1}{|c|}{Access}	& \multicolumn{1}{c|}{Triage}	& \multicolumn{1}{c|}{Visit}	& \multicolumn{1}{c|}{ConsultancyBegin}	& \multicolumn{1}{c|}{R1}			& \multicolumn{1}{c|}{R2}			& \multicolumn{1}{c|}{R3}			& \multicolumn{1}{c|}{R4}		& \multicolumn{1}{c|}{$\tau$}	& \multicolumn{1}{c|}{$\tau$}	& \multicolumn{1}{c|}{$\tau$}	& \multicolumn{1}{c|}{ConsultancyEnd}	& \multicolumn{1}{c|}{$\tau$}			& \multicolumn{1}{c|}{Dismissal}	& \multicolumn{1}{c|}{Exit}		\\ 
			\multicolumn{1}{|c|}{$t_1$}		& \multicolumn{1}{c|}{$t_2$}	& \multicolumn{1}{c|}{$t_8$}	& \multicolumn{1}{c|}{$t_9$}				& \multicolumn{1}{c|}{$t_3$}		& \multicolumn{1}{c|}{$t_4$}		& \multicolumn{1}{c|}{$t_5$}		& \multicolumn{1}{c|}{$t_6$}	& \multicolumn{1}{c|}{}			& \multicolumn{1}{c|}{}			& \multicolumn{1}{c|}{}			& \multicolumn{1}{c|}{$t_{10}$}				& \multicolumn{1}{c|}{$t_{14}$}			& \multicolumn{1}{c|}{$t_{15}$}		& \multicolumn{1}{c|}{$t_{16}$}	\\ 
		\end{tabular}
	\end{adjustbox}
	\normalsize
\end{table}

Uncertain alignments provide novel insights, not obtainable through existing conformance techniques. The process owner can utilize these results to gain insights and decide actions in regard of the process. In situations where quantified uncertainty is present or can be uncovered using domain knowledge in a pre-processing step, the potential violation shown in the worst-case scenario for traces such as the one in Table~\ref{table:uncertaintraceicu1} can be investigated, as well as the source of said uncertainty; the process owner can, furthermore, decide whether the consequences and the likelihood of the worst-case scenario are indicative of a need for a process restructuration, or whether the risk of such potential violation of the normative process model are not critical for the process execution. Conversely, if uncertainty in the event log remains non-quantified and the affected trace is treated as a regular process trace, the subsequent analysis will only consider one possible realization of the uncertain trace, possibly sampled at random. In this case, taking process management decisions that account for the best- and worst-case scenarios is not possible.

Lastly, it is important to notice an additional implication of the qualitative experiment described in this section. For the events $e_8$ through $e_{11}$ of the trace in Table~\ref{table:uncertaintraceicu2} we determined suitable bounds for uncertain timestamps through domain knowledge. In absence of such domain knowledge, it is still possible to apply process mining techniques for uncertain data to traces with missing timestamps: the lower (resp., upper) bound of such timestamps can be set to be smaller (resp., larger) than any other timestamp appearing in the uncertain trace. This models an event that, in the real process, might have occurred in any point in the trace. Thus, the resulting pre-processed event will be able to be considered by process mining techniques operating on uncertain events such as the alignment technique presented in this paper\footnote{This is limited to techniques that only consider the control-flow perspective of event data. Additional perspectives might not be modeled by this pre-processing technique (namely, in this case, the time/performance perspective).}. An analogous pre-processing procedure can be utilized on events with a missing activity label, by assigning an uncertain label containing all labels appearing in the event log. While these pre-processing techniques allow to apply process mining techniques to traces and events with missing attributes, it is important to bear in mind the consequences this might have in terms of performance. As demonstrated by the results of quantitative performance experiments shown in Figure~\ref{fig:num_real_unc}, a small percentage of uncertain events in a log induces a large amount of realizations. Thus, modeling missing timestamps or activity labels through uncertainty without restricting them with domain knowledge might be unfeasible when applied to substantial amounts of event data.

\section{Related Work}\label{sec:related}
This section discusses existing literature relevant to the problem of computing a conformance score between historical event data and a process model, as well as research addressing types of anomalies in recorded data similar to the notion of uncertainty presented in this paper.

\subsection{Conformance Checking}
The discipline of conformance checking, a subfield of process mining, is concerned with defining metrics to compare how well an event log matches a given process model. The input for this task consists of an execution log and a process model (most commonly a labeled Petri net) and the output is a measurement of the distance -- that is, the deviation -- between the model and the log, or the traces that compose the log. The two main goals of conformance checking are measuring the quality of a process discovery algorithm by comparing the discovered process model with the source event log, to verify the extent to which the model fits the log; and comparing an execution log with a normative process model (often defined partially or completely by hand) in order to verify the deviations between the rules governing the process and the tasks carried out in reality. Often, the conformance measure defined between logs (or traces) and models includes not only a distance in absolute terms, but also an indication of where and what deviated from the norm in the process. Conformance checking was introduced by Rozinat and van der Aalst~\cite{rozinat2008conformance}, who obtained a conformance measure by tracking counts of tokens during replay of traces in a Petri net. Despite the elevated computational complexity, state-of-the-art approaches are mostly based on alignments, introduced by Adriansyah et al.~\cite{adriansyah2010towards}.

The topic of conformance checking includes previous work that examines concepts connected to probability on the model side. The stochastic Petri net is an important extension of the Petri net model which probabilistically describes the time distance between the activation of transitions. Richter et al.~\cite{richter2020tade} utilize these models to extend conformance checking so the conformance score of event data can be finely tuned to account for deviations in the time dimension with respect to a reference stochastic Petri net. Another formalism involving probabilities is the Fork/Join network, a scheduling model that can represent the actions of resources in a process, complete with probabilities. More recently, Leemans et al.~\cite{leemans2021stochastic} devised a conformance checking technique able to measure deviations between stochastic Petri nets and event logs. Their method hinges on converting both the stochastic model and the event log in a so-called stochastic language, i.e., a probability distribution over process variants. They then employ the Earth Movers' Distance (EMD) to compute the difference between stochastic languages in term of distance between probability distributions. This allows to account for routing probabilities in the model, improving the reliability of conformance scores.
	
Senderovich et al.~\cite{senderovich2016conformance} show how to discover Fork/Join networks from an event log and a corresponding schedule for the process -- i.e., a description of the tasks involved in a process and an assignment between agents and tasks -- and utilize them to perform conformance checking. This allows to measure the predictive capabilities of the reference schedule -- or, alternatively, to quantify the deviation from the schedule present in an historical event log. Additionally, the authors complement this conformance checking approach on schedules with a process improvement algorithm that shortens the expected delay between tasks in the process.

It is important, however, to note that stochastic conformance checking is a concept that fundamentally differs from uncertainty as presented in this paper. It is essential to understand that strong uncertainty involves non-determinism, and the behavior contained in a strongly uncertain trace is completely probability-agnostic. Moreover, existing approaches for stochastic conformance checking assume the presence of probability information on the reference model, which is then compared with classic process traces; conversely, uncertainty specifically considers anomalies within recorded data, regardless of the nature or the semantics of the corresponding process model.

\subsection{Event Data Uncertainty}\label{sec:related_unc}
As mentioned, the occurrence of data containing uncertainty -- in a broad sense -- is common both in more classic disciplines like statistics and data mining~\cite{han2011data} and in process mining~\cite{van2011process}; and logs that show an explicit uncertainty in the control flow perspective can be classified in the lower levels of the quality ranking proposed in the process mining manifesto.

To historically position the topic of uncertain data, let us mention some previous work in the domain of data mining. A survey work offering a panoramic view of mining uncertain data is the one by Aggarwal and Philip~\cite{aggarwal2008survey}, which focuses with particular attention on the problem of uncertain data querying. Such data is represented on the basis of probabilistic databases~\cite{suciu2011probabilistic}, a foundational notion in the setting of uncertain data mining. A branch of data mining particularly related to process mining is frequent itemsets mining: an efficient algorithm to search for frequent itemsets over uncertain data, the U-Apriori, have been presented by Chui et al.~\cite{chui2007mining}.

Within process mining, there exist various techniques to deal with a kind of uncertainty different, albeit closely related, from the one that we analyze here: missing or incorrect data. This can be considered as a form of non-explicit uncertainty: no measure or indication on the nature of the uncertainty is given in the event log. The work of Suriadi et al.~\cite{suriadi2017event} provides a taxonomy for such issues in event logs, laying out a series of data patterns that model errors in process data. In these cases, and if this behavior is infrequent enough to allow the event log to remain meaningful, the most common way for existing process mining techniques to deal with missing data is by filtering out the affected traces and performing discovery and conformance checking on the resulting filtered event log. A case study illustrating such situation is, e.g., the work of Benevento et al.~\cite{benevento2019evaluating}. While filtering out missing values is straightforward, various methodologies of event log filtering have been proposed in the past to solve the problem of incorrect event attributes: the filtering can take place thanks to a reference model, which can be given as process specification~\cite{wang2015cleaning}, or from information discovered from the frequent and well-formed traces of the same event log; for example extracting an automaton from the frequent traces~\cite{conforti2017filtering}, computing conditional probabilities of frequent sequences of activities~\cite{sani2017improving}, or discovering a probabilistic automaton~\cite{van2018filtering}. In the latter cases, the noise is identified as infrequent behavior.

Some previous work attempt to repair the incorrect values in an event log. Conforti et al.~\cite{conforti2020automatic} propose an approach for the restoration of incorrect timestamps based on a log automaton, that repairs the total ordering of events in a trace based on correct frequent behavior. Fani Sani et al.~\cite{sani2018repairing} define outlier behavior as the unexpected occurrence of an event, the absence of an event that is supposed to happen, and the incorrect order of events in the trace; then, they propose a repairing method based on probabilistic analysis of the context of an outlier (events preceding or following the anomalous event). Again, both of these methods define anomalous/incorrect behavior on the basis of the frequency of occurrence.

The definition of uncertainty on activity labels as defined in the taxonomy of Section~\ref{sec:taxonomy} has not been, to the best of our knowledge, previously employed in the field of process mining. There are, however, related examples of anomalies or uncertainties on activity labels of events: for instance, the problem of matching event identifiers to normative activity labels~\cite{baier2013bridging}. In this case, an event is associated with only one activity label, but this association is not known. There are a number of techniques to estimate the correct association, included some that consider the data perspective, together with the control flow perspective~\cite{senderovich2016road}. Using this setting, van der Aa et al.~\cite{van2019efficient} proposed a technique to estimate bounds of conformance scores for event logs with unknown or partially known event-to-activity mapping. Another related domain is the many-to-one abstraction from low-level events to a higher order of activity labels, which can be performed via clustering events in minimal conflict groups~\cite{gunther2006mining} or representing low-level patterns with data Petri nets which then discovers high-level activities by matching patterns through alignments~\cite{mannhardt2016low}.

A kind of anomaly in event data which is even more related to uncertainty as discussed in this paper is incompleteness in the order of events in a trace. This occurs when total ordering among events is lost or not available, and only a partial order is known. In the field of concurrent and distributed systems, the absence of a total order among logged activities has historically been relevant by virtue of being both caused by, and a necessary condition for, the presence of concurrency in a system (refer e.g. to Beschastnikh et al.~\cite{beschastnikh2011mining}). An important concept at the base of this paper is the representation of uncertainties in the timestamp dimension through directed acyclic graphs, which express these partial orders. This intuition was first presented by Lu et al.~\cite{lu2014conformance}, also in the context of conformance checking, in order to produce partially ordered alignments. More recently, van der Aa et al.~\cite{van2020partial} proposed a technique to resolve such order uncertainty, through estimates based on probabilistic inference aided by a normative process model.

In process mining, a notion well known for a long time is the fact that in many cases the definition of the case is not part of the normative information immediately accessible to the process analyst, so there needs to be a decision on which attribute or attributes constitutes the case of the process. In some cases, multiple definitions of cases are possible and analysis on a subset of them is desirable. This specific setting, which can be interpreted as uncertainty on the case notion, has a long history both in terms of mathematical formalization and in terms of implementation and definition of data standards. For an introduction to this subfield of process mining we refer to~\cite{van2019object}.

This paper presents an extended version of the preliminary analysis on uncertain event data in process mining shown in~\cite{pegoraro2019mining}, in which we presented a preliminary description of uncertain event data and their taxonomy, as well as a description of an approach to find upper and lower bound for the conformance score of an uncertain process trace through alignments. We elaborate on this previous work adding an extended formalization, proving theorems on uncertainty in process mining, and reporting on new experiments. The framework for uncertain data proposed in this paper has also been expanded by providing an algorithm capable of process discovery on uncertain event data through the definition of directly-follows relationship in uncertain settings and the computation of an uncertain directly-follows graph, which enables process discovery techniques~\cite{pegoraro2019discovering}. On the topic of efficient uncertain data management, we presented an improved algorithm that allows to preprocess uncertain traces into behavior graphs in quadratic time, enabling fast uncertainty analysis~\cite{pegoraro2020efficient}. The exploration of uncertain event data can also be facilitated by a memory-efficient representation method and the definition of the concept of uncertain process variants~\cite{pegoraro2020efficient2}.

Lastly, it is important to mention some of the technological advancements that decrease the likelihood of the presence of uncertainty in the data. As discussed in Section~\ref{sec:introduction}, some of the most prominent causes of uncertainty are the human factor involved in the process, and the intrinsic limitations of legacy information systems. Besides more classical concepts like workflow automation~\cite{stohr2001workflow} and the deployment of process-aware information systems~\cite{dumas2005process}, a recent innovation that aims to mitigate both problems is \emph{Robotic Process Automation} (RPA)~\cite{aalst2018robotic}, a technology that aids user operations in processes by learning repeated patterns of actions, and subsequently automate them, while interacting with human operators through the same GUIs they are utilizing. Introducing a high level of automation within the process naturally helps towards the accurate recording of process data, especially if such automation assists the human agents involved in the process.

\section{Conclusion}\label{sec:conclusion}
As the need to quickly and effectively analyze process data has arisen in the recent past and is growing to this day, many new types of information regarding events are recorded; this calls for new techniques able to provide an adequate interpretation of the new data. Not only more and more event data is available to the analyst, but these data are accessible in association with a wealth of information and meta-information about the process, the resources that executed activities, data about the outcome of those actions, and many other types of knowledge about the nature of events, activities, and the process as a whole. In this paper, we presented a new paradigm for process mining applied to event data: explicit uncertainty. We described the possible form it can assume, building a taxonomy of different types of uncertainty, and we provided examples of how uncertainty can originate in a process, and how uncertainty information can be inferred from the available data and from domain knowledge provided by process experts. We then designed a framework to define the various flavors of uncertainty shown in the taxonomy. Then, in order to assess the practical applications of the uncertainty framework, we applied it to a well-consolidated technique for conformance checking: aligning data to a reference Petri net. This application of uncertainty analysis is integrated by theorems that prove the correctness of the techniques developed and illustrated here within the framework previously described. The results can provide insights on the possible violations of process instances recorded with uncertainty against a normative model. The behavior net provides an efficient way to compute the lower bound for the conformance cost -- i.e., the best-case scenario for conformity of uncertain process data -- with a large improvement in time performance with respect to a brute-force procedure.

The approaches shown here can be extended in a number of ways. From a performance perspective, to improve the usability of alignments over uncertainty the computation of the upper bound of the conformance cost should either be optimized, or replaced by an approximate algorithm. Another direction for future work is extending the conformance checking technique to logs with weak uncertainty, weighting the deviation by means of the probability distributions attached to activities, timestamps and indeterminate events. This includes the case in which probability distributions contained in weakly uncertain events are not necessarily independent, or where the assumption of independence is unrealistic for the process being analyzed. Furthermore, a limitation affecting the techniques presented in this paper is that using a graphical representation in lieu of process traces requires to process the entire trace. This implies that uncertain alignments can only be applied to data available in batches, while they do not support event data in streams. Future research might include the development of graphless (i.e., not reliant on graph structures) representations towards online process mining over uncertain event data.

Additionally, investigation on real-life data is an important milestone for this line of research, and it is vital to analyze in depth a complete use case in real life of process mining in the presence of uncertain event data.

\section*{Acknowledgements}
We thank the Alexander von Humboldt (AvH) Stiftung for supporting our research interactions. We acknowledge Elisabetta Benevento for her valuable input.




\section*{References}
\bibliographystyle{elsarticle-num} 
\bibliography{bibliography}


%
%
%
\end{document}